\documentclass{article}

\usepackage{microtype}
\usepackage{graphicx}
\usepackage{booktabs} 

\usepackage{amssymb}
\usepackage{amsmath}
\usepackage{amsfonts}
\usepackage{amsthm}
\usepackage{xcolor}
\usepackage{caption}
\usepackage{subcaption}
\usepackage{float}
\usepackage{multirow}
\usepackage{natbib}
\usepackage[utf8]{inputenc}
\usepackage[english]{babel}

\newtheorem{proposition}{Proposition}
\newtheorem{definition}{Definition}

\usepackage{hyperref}

\usepackage[accepted]{icml2020}

\icmltitlerunning{LowFER: Low-rank Bilinear Pooling for Link Prediction}

\begin{document}

\twocolumn[
\icmltitle{LowFER: Low-rank Bilinear Pooling for Link Prediction}




\begin{icmlauthorlist}
\icmlauthor{Saadullah Amin}{dfki}
\icmlauthor{Stalin Varanasi}{dfki,uds}
\icmlauthor{Katherine Ann Dunfield}{dfki,uds}
\icmlauthor{G\"unter Neumann}{dfki,uds}
\end{icmlauthorlist}

\icmlaffiliation{dfki}{German Research Center for Artificial Intelligence (DFKI), Saarbrücken, Germany}
\icmlaffiliation{uds}{Department of Language Science and Technology, Saarland University, Saarbrücken, Germany}

\icmlcorrespondingauthor{Saadullah Amin}{saadullah.amin@dfki.de}

\icmlkeywords{Knowledge Graph Completion, Link Prediction, Bilinear Models, Statistical Relation Learning, Bilinear Pooling, Low-rank Factorization, Tucker decomposition}

\vskip 0.3in
]



\printAffiliationsAndNotice{}  


\begin{abstract}

Knowledge graphs are incomplete by nature, with only a limited number of observed facts from the world knowledge being represented as structured relations between entities. To partly address this issue, an important task in statistical relational learning is that of \textit{link prediction} or \textit{knowledge graph completion}. Both linear and non-linear models have been proposed to solve the problem. Bilinear models, while expressive, are prone to overfitting and lead to quadratic growth of parameters in number of relations. Simpler models have become more standard, with certain constraints on bilinear map as relation parameters. In this work, we propose a factorized bilinear pooling model, commonly used in multi-modal learning, for better fusion of entities and relations, leading to an efficient and constraint-free model. We prove that our model is fully expressive, providing bounds on the embedding dimensionality and factorization rank. Our model naturally generalizes Tucker decomposition based TuckER model, which has been shown to generalize other models, as efficient low-rank approximation without substantially compromising the performance. Due to low-rank approximation, the model complexity can be controlled by the factorization rank, avoiding the possible cubic growth of TuckER. Empirically, we evaluate on real-world datasets, reaching on par or state-of-the-art performance. At extreme low-ranks, model preserves the performance while staying parameter efficient.

\end{abstract}


\section{Introduction}

Knowledge graphs (KGs) are large collections of structured knowledge, organized as subject and object entities and relations, in the form of fact triples $<$\textit{sub}, rel, \textit{obj}$>$. The usefulness of knowledge graphs, however, is affected primarily by their incompleteness. The task of \textit{link prediction} or \textit{knowledge graph completion} (KGC) aims to infer missing facts from existing ones, by essentially \textit{scoring} a relation and entities triple for use in predicting its validity, and thereby avoiding the cost and time of extending knowledge graphs manually. To accomplish this, several models have been proposed, including linear and non-linear models. Bilinear models have additionally been used in multi-modal learning due to their expressive nature, where the fusion of features from different modalities plays a key role towards the performance of a model, with concatenation or element-wise summation being commonly used fusion techniques. The underlying assumption is that the distributions of features across modalities may vary significantly, and the representation capacity of the fused features may be insufficient, therefore limiting the final prediction performance \cite{yu2017multi}. In this work, we apply this assumption to knowledge graphs by considering that the \textit{entities} and \textit{relations} come from different multi-modal distributions and good fusion between them can potentially construct a KG. 

A major drawback of using bilinear modeling methods is the quadratic growth of parameters, which results in high computational and memory costs and risks overfitting. In multi-modal learning, \textit{factorization} techniques have therefore been researched to address these challenges \cite{kim2016hadamard,fukui2016multimodal,yu2017multi,ben2017mutan,li2017factorized, liu2018efficient}, and \textit{constraints} based bilinear maps have become a more prevalent standard in link prediction \cite{yang2014embedding,trouillon2016complex,kazemi2018simple}. Applying constraints can be seen as hard regularization since it allows for incorporating background knowledge \cite{kazemi2018simple}, but restricts the learning potential of the model due to limited parameter sharing \cite{balazevic2019tucker}. We focus on a constraint-free approach, using the low-rank factorization of bilinear models, as it offers flexibility and generalizes well, naturally leading to other models under certain conditions. Our work extends the multi-modal factorized bilinear pooling (MFB) model, introduced by \citet{yu2017multi}, and applies it to the link prediction task. 

Our contributions are outlined as follows:

\begin{itemize}
  \item We propose a simple and parameter efficient linear model by extending multi-modal factorized bilinear (MFB) pooling \cite{yu2017multi} for link prediction.
  \item We prove that our model is \textit{fully expressive} and provide bounds on entity and relation embedding dimensions, along with the factorization rank.
  \item We provide relations to the family of bilinear link prediction models and Tucker decomposition \cite{tucker1966some} based TuckER model \cite{balazevic2019tucker}, generalizing them as special cases. We also show the relation to 1D convolution based HypER model \cite{balavzevic2019hypernetwork}, bridging the gap from bilinear to convolutional link prediction models.
  \item On real-world datasets the model achieves on par or state-of-the-art performance, where at extreme low-ranks, with limited number of parameters, it outperforms most of the prior arts, including deep learning based models.   
\end{itemize}


\section{Related Work}

Given a set of entities $\mathcal{E}$ and relations $\mathcal{R}$ in a knowledge graph $\mathcal{KG}$, the task of link prediction is to assign a score $s$ to a triple $(e_s, r, e_o)$:
\begin{equation*}
    s = f(e_s, r, e_o)
\end{equation*}
where $e_s \in \mathcal{E}$ is the \textit{subject} entity, $e_o \in \mathcal{E}$ is the \textit{object} entity and $r \in \mathcal{R}$ is the \textit{relation} between them. The scoring function $f$ estimates the general binary tensor $\mathbf{T} \in |\mathcal{E}| \times |\mathcal{R}| \times |\mathcal{E}|$, by assigning a score of $1$ to $\mathbf{T}_{ijk}$ if relation $r_j$ exists between entities $e_i$ and $e_k$, $0$ otherwise. The scoring function can be a linear or non-linear model, trained to predict true triples in a $\mathcal{KG}$.

Deep learning based scoring functions such as ConvE \cite{dettmers2018convolutional} and HypER \cite{balavzevic2019hypernetwork} use 2D and 1D convolution on subject entity and relation representations respectively. Both perform well in practice and are efficient, but the former lacks direct interpretation, whereas the latter has shown to be related to tensor factorization. Transitional methods \cite{bordes2013translating,wang2014knowledge,ji2015knowledge,lin2015modeling,nguyen2016stranse,feng2016knowledge} use additive dissimilarity scoring functions, whereby they differ in terms of the constraints applied to the projection matrices. While interpretable, they are theoretically limited as they have shown to be not \textit{fully expressive} \cite{wang2018multi,kazemi2018simple}. There are several other related works \cite{nickel2016holographic,das2017go,yang2017differentiable,shen2018m,schlichtkrull2018modeling,ebisu2018toruse,sun2019rotate}, but we will mainly focus on different types of linear models here, as they are more relevant to our work. 

All discussed linear models can be seen as a decomposition of the tensor $\mathbf{T}$, using different factorization methods. One way to factorize this tensor is to factorize its slices in the relation dimension with DEDICOMP \cite{harshman1978models}. RESCAL \cite{nickel2011three}, a relaxed version of DEDICOMP, decomposes using a scoring function that consists of a bilinear product between subject and object entity vectors with a relation specific matrix. RESCAL, however, tends to overfit due to the quadratic growth of parameters in number of relations. Others use Canonical Polyadic decomposition (CPD or simply CP) \cite{hitchcock1927expression,harshman1994parafac} to factorize the binary tensor. In CP, each value in the tensor is obtained as a sum of multiple Hadamard products of three vectors, representing subject, object and relation. DistMult \cite{yang2014embedding}, equivalent to INDSCAL \cite{carroll1970analysis}, is as such and uses a diagonal relation matrix, unlike RESCAL, to account for overfitting. ComplEx \cite{trouillon2016complex,trouillon2017complex} uses complex valued vectors for entities and relations to explicitly model asymmetric relations. SimplE \cite{kazemi2018simple} extends CP by introducing two vectors (\textit{head} and \textit{tail}) for each entity and two for relations (including the inverse). Tucker decomposition \cite{tucker1966some} based TuckER \cite{balazevic2019tucker} learns a 3D core tensor, which is multiplied with a matrix along each mode to approximate the binary tensor. A key difference between CP based methods and TuckER is that it learns representations not only via embeddings, but also through a shared core tensor.


\section{Model}

Downstream performance for tasks such as visual question answering strongly depends on the multi-modal fusion of features to leverage the heterogeneous data \cite{liu2018efficient}. Bilinear models are expressive as they allow for pairwise interactions between the feature dimensions but also introduce huge number of parameters that lead to high computational and memory costs and the risk of overfitting \cite{fukui2016multimodal}. Substantial research has therefore focused on efficiently computing the bilinear product. In multi-modal compact bilinear (MCB) pooling \cite{gao2016compact,fukui2016multimodal}, authors employ a sampling based approximation that uses the property that the tensor sketch projection \cite{charikar2004finding,pham2013fast} of the outer product of two vectors can be represented as their sketches convolution. Multi-modal low-rank bilinear (MLB) pooling \cite{kim2016hadamard} uses two low-rank projection matrices to transform the features from the original space to a common space, followed by the Hadamard product, which was later generalized by the multi-modal factorized bilinear (MFB) pooling \cite{yu2017multi}. Our work is based on the MFB model but can also be seen as related to \citet{liu2018efficient}. In contrast to KGC bilinear models, these bilinear models allow for parameter sharing and generally, are constraint-free.


\subsection{Multi-modal Factorized Bilinear Pooling (MFB)} \label{section:mfb}

Given two feature vectors $\mathbf{x} \in \mathbb{R}^{m}$, $\mathbf{y} \in \mathbb{R}^{n}$ and a bilinear map $\mathbf{W} \in \mathbb{R}^{m \times n}$, the bilinear transformation is defined as $z = \mathbf{x}^{T}\mathbf{W}\mathbf{y} \in \mathbb{R}$. To obtain a vector in $\mathbb{R}^{o}$, $o$ such maps have to be learned (e.g. in RESCAL these would be relation specific matrices), resulting in large number of parameters. However, $\mathbf{W}$ can be factorized into two low-rank matrices:
\begin{align*}
    z = \mathbf{x}^{T}\mathbf{U}\mathbf{V}^{T}\mathbf{y} = \mathbf{1}^{T}(\mathbf{U}^{T}\mathbf{x} \circ \mathbf{V}^{T}\mathbf{y})
\end{align*}
where $\mathbf{U} \in \mathbb{R}^{m \times k}$, $\mathbf{V} \in \mathbb{R}^{n \times k}$, $k$ is the factorization rank, $\circ$ is the element-wise product of two vectors and $\mathbf{1} \in \mathbb{R}^{k}$ is vector of all ones. Therefore, to obtain a output feature vector $\mathbf{z} \in \mathbb{R}^{o}$, two 3D tensors are required, $\mathbf{W}_{x} = [\mathbf{U}_1, \mathbf{U}_2, ..., \mathbf{U}_o] \xrightarrow{\textit{reshape}} \mathbf{W}_{x}^{'}$ and $\mathbf{W}_{y} = [\mathbf{V}_1, \mathbf{V}_2, ..., \mathbf{V}_o] \xrightarrow{\textit{reshape}} \mathbf{W}_{y}^{'}$, where $\mathbf{W}_{x} \in \mathbb{R}^{m \times k \times o}$, $\mathbf{W}_{y} \in \mathbb{R}^{n \times k \times o}$ are 3D tensors and $\mathbf{W}_{x}^{'} \in \mathbb{R}^{m \times ko}$, $\mathbf{W}_{y}^{'} \in \mathbb{R}^{n \times ko}$ are their reshaped 2D matrices respectively. The final (fused) vector $\mathbf{z}$ is then obtained by summing non-overlapping windows of size $k$ over the Hadamard product of projected vectors using $\mathbf{W}_{x}^{'}$ and $\mathbf{W}_{y}^{'}$:
\begin{align}
    \mathbf{z} &= \text{SumPool}(\mathbf{W}_{x}^{'T}\mathbf{x} \circ \mathbf{W}_{y}^{'T}\mathbf{y}, k) \label{eq:mfb}
\end{align}
At $k=1$, MFB reduces to MLB, which converges slowly, and MCB requires very high-dimensional vectors to perform well \cite{yu2017multi}. Further, MFB significantly lowers the number of parameters with low-rank factorized matrices and leads to better performance.


\subsection{Low-rank Bilinear Pooling for Link Prediction}

Consider that \textit{entities} and \textit{relations} are not intrinsically bound and come from two different modalities, such that \textit{good} fusion between them can potentially result in a knowledge graph of fact triples. Entities and relations can be shown to possess certain properties that allow them to function similarly to others within the same modality. For example, the relation \textit{place-of-birth} shares inherent properties with the relation \textit{residence}. As such, similar entity pairs can yield similar relations, given appropriate shared properties. Like in multi-modal auditory-visual fusion, where the sound of a roar may better predict a resulting image within the distribution of animals that roar, a relation such as \textit{place-of-birth}, can better predict an entity pair within a distribution of (\textit{person}, \textit{place}) entity pairs. In link prediction, we assume that the latent decomposition with MFB can help the model capture different aspects of interactions between an entity and a relation, which can lead to better scoring with the missing entity. We therefore, apply the \textbf{Low}-rank \textbf{F}actorization trick of bilinear maps with $k$-sized non-overlapping summation pooling (section \ref{section:mfb}) to \textbf{E}ntities and \textbf{R}elations (LowFER). 


\begin{figure}[!tb] \label{fig:fig_lowfer_overview}
    \includegraphics[width=0.82\linewidth]{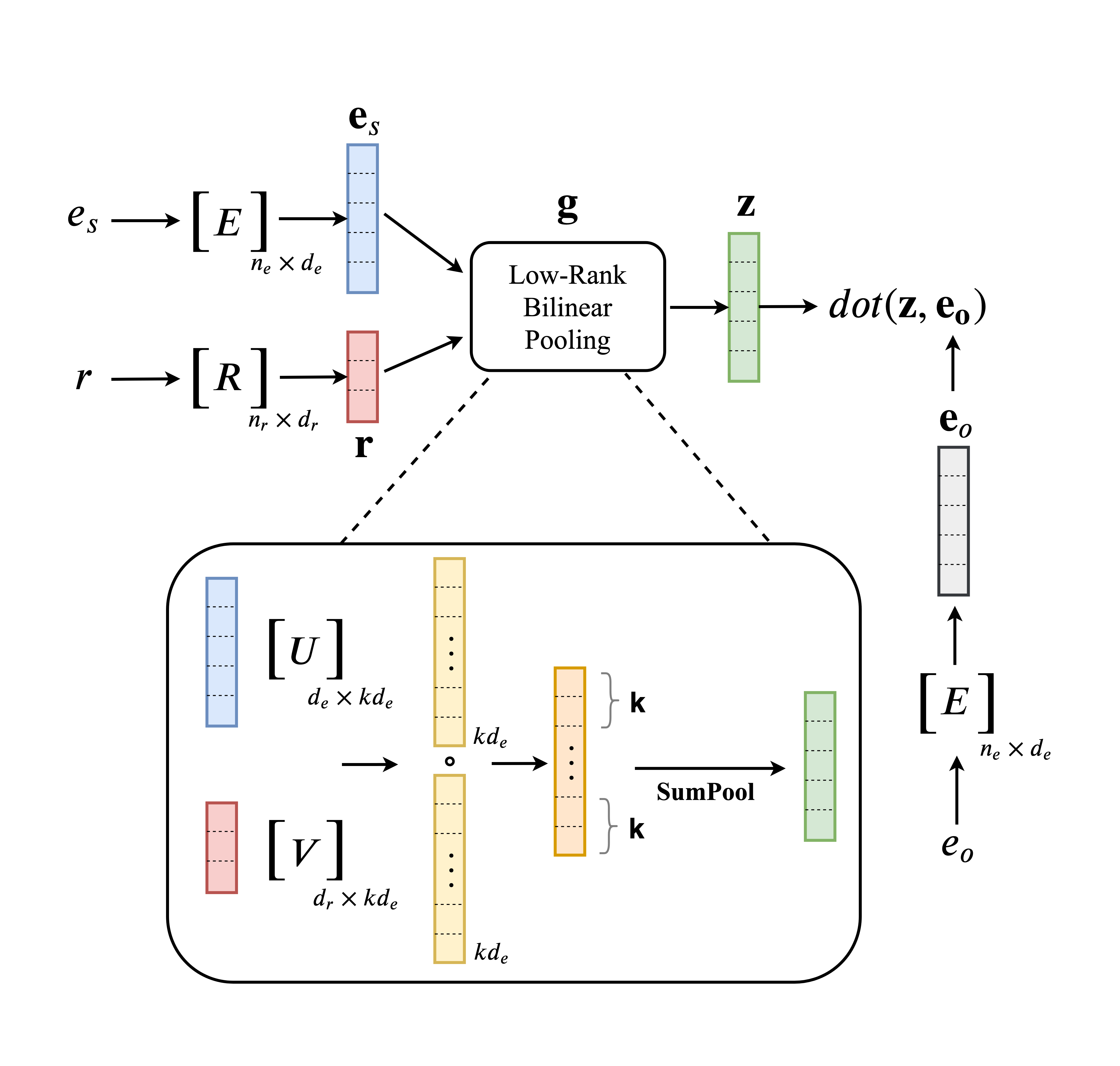}
    \centering
    \vspace{-0.5cm}
    \caption{Overview of the LowFER model. For an input tuple $(e_s, r)$ and target entity $e_o$, we first get entity vectors $\mathbf{e}_s, \mathbf{e}_o \in \mathbb{R}^{d_e}$ from entity embedding matrix $\mathbf{E} \in \mathbb{R}^{n_e \times d_e}$ and relation vector $\mathbf{r} \in \mathbb{R}^{d_r}$ from relation embedding matrix $\mathbf{R} \in \mathbb{R}^{n_r \times d_r}$, where $n_e$ and $n_r$ are number of entities and relations in $\mathcal{KG}$. LowFER projects $\mathbf{e}_s$ and $\mathbf{r}$ into a common space $\mathbb{R}^{kd_e}$ followed by Hadamard product and $k$-summation pooling, where $k$ is the factorization rank. The output vector $\mathbf{z}$ is then matched against target entity $\mathbf{e}_o$ to give final score.}
\end{figure}

More formally, for an entity $e \in \mathcal{E}$, we represent its embedding vector $\mathbf{e}$ of $d_e$ dimension as a look-up from entity embedding matrix $\mathbf{E} \in \mathbb{R}^{n_e \times d_e}$, where $n_e = |\mathcal{E}|$. Similarly, for a relation $r \in \mathcal{R}$, we represent its embedding vector $\mathbf{r}$ of $d_r$ dimension as a look-up from relation embedding $\mathbf{R} \in \mathbb{R}^{n_r \times d_r}$, where $n_r = |\mathcal{R}|$. Then, for a given triple $(e_s, r, e_o)$, we define our scoring function as:
\begin{equation}
    f(e_s, r, e_o) := \mathbf{g}(e_s, r) \cdot \mathbf{e}_{o} = \mathbf{g}(e_s, r)^{T}\mathbf{e}_{o} \label{eq:main_scoring_func_lowfer}
\end{equation}
where $\mathbf{g}(., .) \in \mathbb{R}^{d_e}$ is a vector valued function of the subject entity vector $\mathbf{e}_s$ and the relation vector $\mathbf{r}$, defined from Eq. \ref{eq:mfb} as:
\begin{equation}
    \mathbf{g}(e_s, r) := \text{SumPool}(\mathbf{U}^{T}\mathbf{e}_{s} \circ \mathbf{V}^{T}\mathbf{r}, k) \label{eq:sumpool_lowfer}
\end{equation}
where matrices $\mathbf{U} \in \mathbb{R}^{d_e \times kd_e}$ and $\mathbf{V} \in \mathbb{R}^{d_r \times kd_e}$ represent our model parameters. We can re-write the Eq. \ref{eq:sumpool_lowfer} more compactly as: 
\begin{equation}
    \mathbf{g}(e_s, r) = \mathbf{S}^k\text{diag}(\mathbf{U}^T \mathbf{e}_s)\mathbf{V}^T \mathbf{r} \label{eq:compact_g}
\end{equation}
where $\text{diag}(\mathbf{U}^T \mathbf{e}_s) \in \mathbb{R}^{kd_e \times kd_e}$ and $\mathbf{S}^k \in \mathbb{R}^{d_e \times kde}$ is a constant matrix\footnote{Note that at $k=1$, $\mathbf{S}^1 = \mathbf{I}_{d_e}$} such that: 
\begin{equation*}
    \mathbf{S}^{k}_{i,j} =
    \begin{cases}
      1, & \forall  j \in [(i-1)k+1, ik] \\
      0, & \text{otherwise}
    \end{cases}
\end{equation*}
Using this compact notation in Eq. \ref{eq:main_scoring_func_lowfer}, the final scoring function of LowFER is obtained as:
\begin{equation}
    f(e_s, r, e_o) = (\mathbf{S}^k\text{diag}(\mathbf{U}^T \mathbf{e}_s)\mathbf{V}^T \mathbf{r})^{T}\mathbf{e}_{o} \label{eq:final_scoring_func_lowfer} 
\end{equation}


\subsection{Training LowFER}

To train the LowFER model, we follow the setup of \citet{balazevic2019tucker}. First, we apply sigmoid non-linearity after Eq. \ref{eq:final_scoring_func_lowfer} to get the probability $p(y_{(e_s, r, e_o)}) = \sigma(f(e_s, r, e_o))$ of a triple belonging to a $\mathcal{KG}$. Then, for every triple $(e_s, r, e_o)$ in the dataset, a reciprocal relation is added by generating a synthetic example $(e_o, r^{-1}, e_s)$ \cite{dettmers2018convolutional,lacroix2018canonical} to create the training set $\mathcal{D}$. For faster training, \citet{dettmers2018convolutional} introduced 1-N scoring, where each tuple $(e_s, r)$ and $(e_o, r^{-1})$ is simultaneously scored against all entities $e \in \mathcal{E}$ to predict $1$ if $e=e_o$ or $e_s$ respectively and $0$ elsewhere (see \citet{trouillon2017knowledge} and \citet{sun2019rotate} for other methods to collect negative samples). The model is trained with binary cross-entropy instead of margin based ranking loss \cite{bordes2013translating}, which is prone to overfitting for link prediction \cite{trouillon2017complex, kazemi2018simple}. For a mini-batch $\mathcal{B}$ of size $m$ drawn from $\mathcal{D}$, we minimize:
\begin{multline*}
    \min_{\Theta} \frac{1}{m} \sum_{(e, r) \in \mathcal{B}} \frac{1}{n_e} \sum_{i=1}^{n_e}(  y_i \log(p(y_{(e, r, e_i)})) \\ + (1-y_i) \log(1 - p(y_{(e, r, e_i)})) )
\end{multline*}
where $y_i$ is a target label for a given entity-relation pair $(e, r)$ for entity $e_i$, $y_{(e, r, e_i)}$ is the model prediction and $\Theta$ represents model parameters. Following \citet{yu2017multi}, we also apply the power normalization $\mathbf{x} \leftarrow \text{sign}(\mathbf{x})|\mathbf{x}|^{0.5}$ and $l_2$-normalization $\mathbf{x} \leftarrow \mathbf{x}/||\mathbf{x}||$ before summation pooling to stabilize the training from large output values as a result of Hadamard product in Eq. \ref{eq:sumpool_lowfer}.


\begin{table}[!tb]
    \def\arraystretch{1.3}
    \centering
    \caption{Bounds for \textit{fully expressive} linear models, where $n$ is the number of true facts and the trivial bound is given by $n_e^2n_r$.} \label{table:bounds_summary}
    \vspace{0.3cm}
    \resizebox{7.6cm}{!}{
      \begin{tabular}{lc}
        \toprule
        Model & Full Expressibility Bounds \\
        \midrule
        RESCAL \cite{nickel2011three} & $(d_e, d_r)=(n_e,n_e^2)$ \\
        HolE \cite{nickel2016holographic}  & $d_e=d_r=2n_en_r+1$\\
        ComplEx \cite{trouillon2016complex} & $d_e=d_r=n_en_r$ \\
        SimplE \cite{kazemi2018simple} & $d_e=d_r=\text{min}(n_en_r, n+1)$ \\
        TuckER \cite{balazevic2019tucker} & $(d_e,d_r)=(n_e,n_r)$ \\
        LowFER & $(d_e,d_r)=(n_e,n_r) \ \text{for} \ k=\text{min}(n_e, n_r)$ \\
        \bottomrule
      \end{tabular}
    }
    \vspace{-0.2cm}
\end{table}


\section{Theoretical Analysis} \label{sec:theory}


\subsection{Full Expressibility} \label{sec:fully_expressive}

A key theoretical property of link prediction models is their ability to be fully expressive, which we define formally as:


\begin{definition}
Given a set of entities $\mathcal{E}$, relations $\mathcal{R}$, correct triples $\mathcal{T} \subseteq \mathcal{E} \times \mathcal{R} \times \mathcal{E}$ and incorrect triples $\mathcal{T'} = \mathcal{E} \times \mathcal{R} \times \mathcal{E} \setminus \mathcal{T}$, then a model $\mathcal{M}$ with scoring function $f(e_s, r, e_o)$ is said to be fully expressive iff it can accurately separate $\mathcal{T}$ from $\mathcal{T'}$ for all $e_s, e_o \in \mathcal{E}$ and $r \in \mathcal{R}$.
\end{definition}

A \textit{fully expressive} model can represent relations of any type, including \textit{symmetric}, \textit{asymmetric}, \textit{reflexive}, and \textit{transitive} among others. Models such as RESCAL, HolE, ComplEx, SimplE and TuckER have been shown to be \textit{fully expressive} \cite{trouillon2017complex,wang2018multi,kazemi2018simple,balazevic2019tucker}. On the other hand, DistMult is not \textit{fully expressive} as it enforces symmetric relations only. Further, \citet{wang2018multi} showed that TransE is not \textit{fully expressive}, which was later expanded by \citet{kazemi2018simple}, showing that other translational variants, including, FTransE, STransE, FSTransE, TransR and TransH are likewise not \textit{fully expressive}. By the virtue of universal approximation theorem \cite{cybenko1989approximation,hornik1991approximation}, neural networks can be considered \textit{fully expressive} \cite{kazemi2018simple}. Table \ref{table:bounds_summary} summarizes the bounds of linear models that are \textit{fully expressive}. With Proposition \ref{proposition1} (proof in Appendix \ref{appendix:A.1}), we establish that LowFER is \textit{fully expressive} and provide bounds on entity and relation embedding dimensions and the factorization rank $k$.


\begin{figure}[!tb]
    \includegraphics[width=1.0\linewidth]{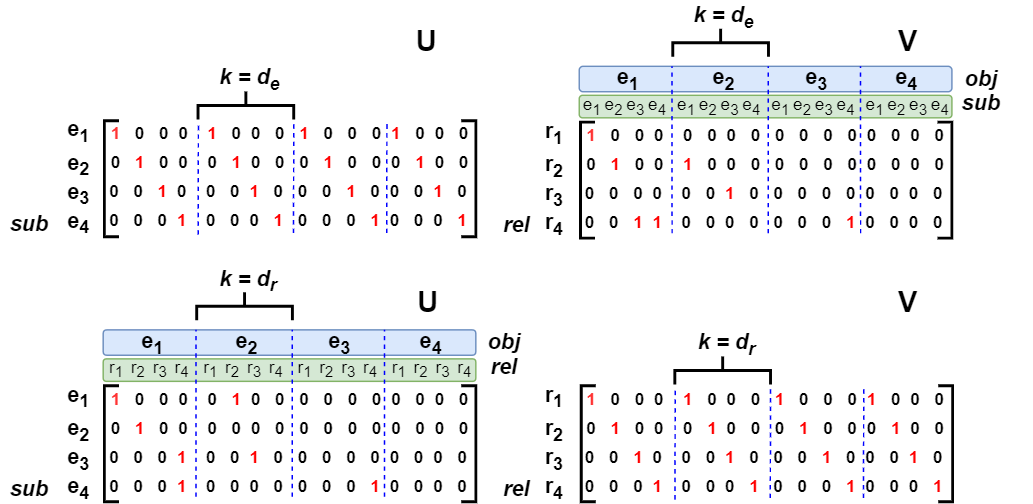}
    \centering
    \caption{LowFER model parameters for a toy dataset under the settings used in Proposition \ref{proposition1}. \textit{Top}: For the case when $k=d_e=n_e$. \textit{Bottom}: For the case when $k=d_r=n_r$.} \label{fig:example_proof}
\end{figure}


\begin{proposition} \label{proposition1}
For a set of entities $\mathcal{E}$ and a set of relations $\mathcal{R}$, given any ground truth $\mathcal{T}$, there exists an assignment of values in the LowFER model with entity embeddings of dimension $d_e=|\mathcal{E}|$, relation embeddings of dimension $d_r=|\mathcal{R}|$ and the factorization rank $k = \min(d_e, d_r)$ that makes it fully expressive.
\end{proposition}

As a given example, consider a set of entities $\mathcal{E} = \{e_1, e_2, e_3, e_4\}$ and relations $\mathcal{R} = \{r_1, r_2, r_3, r_4\}$ such that $r_1$ is reflexive,  $r_2$ is symmetric, $r_3$ is asymmetric, and $r_4$ is transitive, then for ground truth $\mathcal{T} = \{(e_1, r_1, e_1), (e_1, r_2, e_2), (e_2, r_2, e_1), (e_3, r_3, e_2)$ $, (e_4, r_4, e_3), (e_3, r_4, e_1), (e_4, r_4, e_1)\}$ and following the settings in Proposition \ref{proposition1}, Figure \ref{fig:example_proof} shows the model parameters $\mathbf{U}$ and $\mathbf{V}$ for this toy example. Now, consider the case $k=d_e=n_e$, then $\mathbf{U}$ copies each entity vector in $k$-sized slices and $\mathbf{V}$ buckets target entities per relation such that each source entity is distributed into disjoint sets. Note that reshaping $\mathbf{V}$ as 3D tensor of size $n_r \times n_e \times n_e$ and transposing first two dimensions results in binary tensor $\mathbf{T}$. 


\begin{figure*}[!tbp]
    \includegraphics[width=0.95\linewidth]{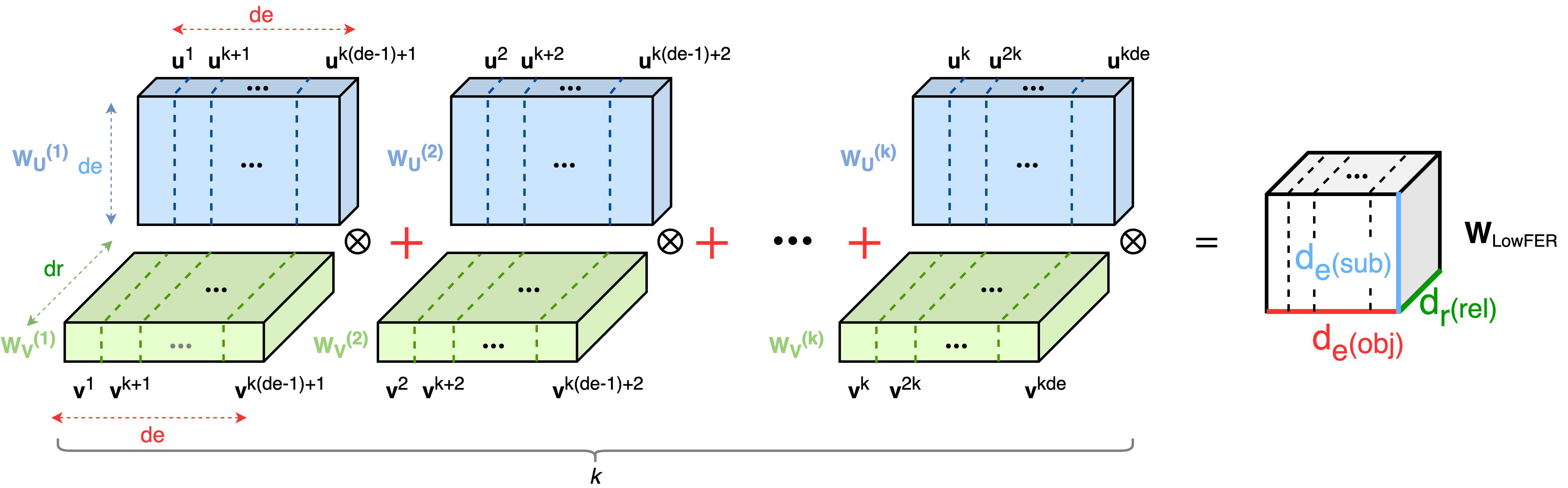}
    \centering
    \caption{Low-rank approximation of the core tensor $\mathcal{W}$ of TuckER \cite{balazevic2019tucker} with LowFER by summing $k$ low-rank 3D tensors, where each tensor is obtained by stacking $d_e$ rank-1 matrices obtained by the outer product of $k$-apart columns of $\mathbf{U}$ and $\mathbf{V}$.} \label{fig:lowfer_core_appx}
\end{figure*}


\subsection{Relation with TuckER} \label{sec:relation_with_tucker}

Initially, it was shown by \citet{kazemi2018simple} that RESCAL, DistMult, ComplEx and SimplE belong to a \textit{family of bilinear models} with different set of constraints. Later, \citet{balazevic2019tucker} established that TuckER generalizes all of these models as special cases. In this section, we will formulate relation between our model and TuckER \cite{balazevic2019tucker}, followed by relations with the \textit{family of bilinear models} in the next section. This provides a unifying view and shows LowFER's ability to generalize. 

TuckER's scoring function is defined as follows \cite{balazevic2019tucker}:
\begin{equation}
    \phi_{t}(e_s, r, e_o) = \mathcal{W} \times_1 \mathbf{e}_s \times_2 \mathbf{r} \times_3 \mathbf{e}_o \label{eq:tucker_scoring_func}
\end{equation}
where $\mathcal{W} \in \mathbb{R}^{d_e \times d_r \times d_e}$ is the core tensor, $\mathbf{e}_s, \mathbf{e}_o \in \mathbb{R}^{d_e}$ and $\mathbf{r} \in \mathbb{R}^{d_r}$ are subject entity, object entity and the relation vectors respectively. $\times_n$ denotes the tensor product along the $n$-th mode. First, note that Eq. \ref{eq:compact_g} can be expanded as:
\begin{equation*}
    \mathbf{S}^k(\mathbf{U}^T \mathbf{e}_s \circ \mathbf{V}^T \mathbf{r}) = \begin{bmatrix}
    \mathbf{e}_{s}^{T}(\sum_{i=1}^{k} \mathbf{u}_{i} \otimes \mathbf{v}_{i})\mathbf{r} \\
    \vdots  \\
    \mathbf{e}_{s}^{T}(\sum_{i=(j-1)k+1}^{jk} \mathbf{u}_{i} \otimes \mathbf{v}_{i})\mathbf{r} \\
    \vdots \\
    \mathbf{e}_{s}^{T}(\sum_{i=k(d_e-1)}^{kd_e} \mathbf{u}_{i} \otimes \mathbf{v}_{i})\mathbf{r}   
\end{bmatrix}
\end{equation*}
where $\mathbf{u}_i \in \mathbb{R}^{d_e}$ and $\mathbf{v}_i \in \mathbb{R}^{d_r}$ are column vectors of $\mathbf{U}$ and $\mathbf{V}$ respectively and $\otimes$ represents the outer product of two vectors. To take the vectors $\mathbf{e}_s$ and $\mathbf{r}$ out, we realize the above matrix operations in a different way. We first create $k$ matrices sliced from $\mathbf{U}$ and $\mathbf{V}$ each, such that each matrix is formed by choosing all adjacent column vectors that are $k$ distance apart in $\mathbf{U}$ (and $\mathbf{V}$), i.e., for the $l$-th slice, we have $\mathbf{W}^{(l)}_{U} = [\mathbf{u}_{l}, \mathbf{u}_{k+l}, ..., \mathbf{u}_{k(d_e-1)+l}] \in \mathbb{R}^{d_e \times d_e}$ and $\mathbf{W}^{(l)}_{V} = [\mathbf{v}_{l}, \mathbf{v}_{k+l}, ..., \mathbf{v}_{k(d_e-1)+l}] \in \mathbb{R}^{d_r \times d_e}$. Taking the column-wise outer product of these sliced matrices forms a 3D tensor in $\mathbb{R}^{d_e \times d_r \times d_e}$. With slight abuse of notation, we also use $\otimes$ to represent this tensor operation. It can be viewed as transforming the matrix obtained by mode-$2$ Khatri-Rao product into a 3D tensor \cite{cichocki2016tensor}. Now consider a 3D tensor $\mathbf{W} \in \mathbb{R}^{d_e \times d_r \times d_e}$ as the sum of these $k$ products:
\begin{equation}
    \mathbf{W} = \sum_{i=1}^{k} \mathbf{W}^{(i)}_{U} \otimes \mathbf{W}^{(i)}_{V} \label{eq:tucker_core_appx}
\end{equation}
Figure \ref{fig:lowfer_core_appx} shows these operations. With this tensor, the scoring function $f$ in Eq. \ref{eq:final_scoring_func_lowfer} can be re-written as TuckER's scoring function as follows:
\begin{equation}
    \hat{\phi}_{t}(e_s, r, e_o) =\mathbf{W} \times_1 \mathbf{e}_s \times_2 \mathbf{r} \times_3 \mathbf{e}_o \label{eq:tucker_approx}
\end{equation}
It should be noted that $\mathbf{W}$ in Eq. \ref{eq:tucker_approx} is obtained as a summation of $k$ low-rank 3D tensors, each of which is obtained by stacking rank-$1$ matrices in contrast to TuckER's core tensor $\mathcal{W}$ in Eq. \ref{eq:tucker_scoring_func}, which can be a full rank 3D tensor. Our model can therefore approximate TuckER and can be viewed as a generalization of TuckER \cite{balazevic2019tucker}. We further show that we can accurately obtain $\mathcal{W}$ with appropriate $\mathbf{W}_U^{(i)}$'s and $\mathbf{W}_V^{(i)}$'s in Eq. \ref{eq:tucker_core_appx} (proof in Appendix \ref{appendix:A.2}).


\begin{proposition} \label{proposition2}
Given a TuckER model with entity embedding dimension $d_e$, relation embedding dimension $d_r$ and core tensor $\mathcal{W}$, there exists a LowFER model with k $<= \min(d_e,d_r)$, entity embedding dimension $d_e$ and relation embedding dimension $d_r$ that accurately represents the former.
\end{proposition}


\begin{figure*}[!tbp]
    \includegraphics[width=1.0\linewidth]{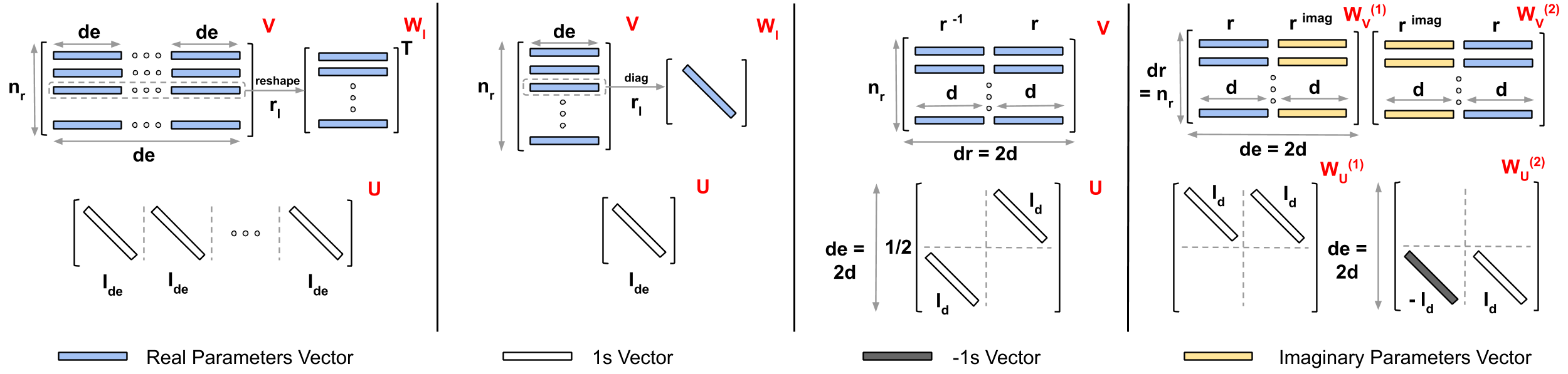}
    \centering
    \caption{Modeling the family of bilinear models with LowFER, (\textit{from left-to-right}): RESCAL \cite{nickel2011three}, DistMult \cite{yang2014embedding}, SimplE \cite{kazemi2018simple} and ComplEx \cite{trouillon2016complex} (see section \ref{sec:relations_with_bilinear_models} for details).}\label{fig:lowfer_rel_with_bilinear_models}
\end{figure*}

LowFER and TuckER parameters grow linearly in the number of entities and relations as $\mathcal{O}(n_ed_e + n_rd_r)$. However, LowFER's shared parameters complexity can be controlled by decoupled low-rank matrices through the factorization rank, making it more flexible, e.g., consider $d=d_e=d_r$, the core tensor $\mathcal{W}$ of TuckER grows as $\mathcal{O}(d^3)$, whereas LowFER grows only as $\mathcal{O}(kd^2)$. As an example, in \citet{lacroix2018canonical} authors used $d_e=d_r=2000$ which would require more than $8$ billion parameters to model with TuckER compared to only $4k$ million for LowFER, with $k$ controlling the growth. More generally, at $k=d_e/2$, LowFER has equal number of parameters as TuckER therefore, we expect similar performance at such rank values. In practice, $k=\{1, 10, 30\}$ performs extremely well (section \ref{sec:link_prediction_results}).


\subsection{Relations with the Family of Bilinear Models} \label{sec:relations_with_bilinear_models}

In this section, we will establish relations between LowFER and other bilinear models. For simplicity, we consider the relation embedding to be a constant matrix $\mathbf{R}=\mathbf{I}_{n_r}$ in all the cases and use $\mathbf{V}$ to model relation parameters. However, the conditions presented here can be extended otherwise, with the remark that they are not unique.

\noindent \textbf{\textsc{RESCAL}} \cite{nickel2011three}: scoring function is defined as:
\begin{equation*}
    \phi_{r}(e_s, r_l, e_o) = \mathbf{e}_s^T \mathbf{W}_l \mathbf{e}_o
\end{equation*}
where $\mathbf{W}_l \in \mathbb{R}^{d_e \times d_e}$ is $l$-th relation matrix. For LowFER to encode RESCAL with Eq. \ref{eq:final_scoring_func_lowfer}, we set $k=d_e$, $d_r=n_r$ and $\mathbf{U} = [\; \mathbf{I}_{d_e} \; | \; \mathbf{I}_{d_e} \; | \; ... \; | \; \mathbf{I}_{d_e} \;] \in \mathbb{R}^{d_e \times d_e^2}$ (block matrix partitioned as $d_e$ identity matrices of size $d_e \times d_e$). This is effectively taking a row $l$ from $\mathbf{V} \in \mathbb{R}^{n_r \times {d_e}^2}$, reshaping it to $d_e \times d_e$ matrix and then taking the transpose to get the equivalent $\mathbf{W}_l$ in RESCAL's scoring function.

\noindent \textbf{\textsc{DistMult}} \cite{yang2014embedding}: scoring function is defined as:
\begin{equation*}
    \phi_{d}(e_s, r_l, e_o) = \mathbf{e}_s^T \text{diag}(\mathbf{w}_l) \mathbf{e}_o
\end{equation*}
where $\mathbf{w}_l \in \mathbb{R}^{d_e}$ is the vector for $l$-th relation.  For LowFER to encode DistMult with Eq. \ref{eq:final_scoring_func_lowfer}, we set $k=1$, $d_r=n_r$ and $\mathbf{U}=\mathbf{I}_{d_e}$. This is effectively taking a row $l$ from $\mathbf{V} \in \mathbb{R}^{n_r \times d_e}$ and creating a diagonal matrix of it to get the equivalent diag$(\mathbf{w}_l)$ in DistMult's scoring function.

\noindent \textbf{\textsc{SimplE}} \cite{kazemi2018simple}: scoring function is defined as:
\begin{equation*}
    \phi_{s}(e_s, r_l, e_o) = \frac{1}{2}( \mathbf{h}_{e_s}^T \text{diag}(\mathbf{r}_l) \mathbf{t}_{e_o} + \mathbf{h}_{e_o}^T \text{diag}(\mathbf{r}^{-1}_l) \mathbf{t}_{e_s})
\end{equation*}
where $\mathbf{h}_{e_s}, \mathbf{h}_{e_o} \in \mathbb{R}^{d}$ are subject, object entities head vectors, $\mathbf{t}_{e_s}, \mathbf{t}_{e_o} \in \mathbb{R}^{d}$ are subject, object entities tail vectors and $\mathbf{r}_l, \mathbf{r}^{-1}_l \in \mathbb{R}^{d}$ are relation and inverse relation vectors. Let $\hat{\mathbf{e}}_s=[\mathbf{t}_{e_s}; \mathbf{h}_{e_s}] \in \mathbb{R}^{2d}$, $\mathbf{e}_o=[\mathbf{h}_{e_o}; \mathbf{t}_{e_o}] \in \mathbb{R}^{2d}$ and $\hat{\mathbf{r}}_l = [\mathbf{r}^{-1}_l; \mathbf{r}_l] \in \mathbb{R}^{2d}$ then SimplE scoring is equivalent to $\frac{1}{2} \hat{\mathbf{e}}_s^T \text{diag}(\hat{\mathbf{r}}_l) \mathbf{e}_o$, where $\hat{\mathbf{e}}_s$ and $\hat{\mathbf{r}}_l$ are obtained by swapping the head, tail vectors in $\mathbf{e}_s=[\mathbf{h}_{e_s}; \mathbf{t}_{e_s}]$ and relation, inverse relation vectors in $\mathbf{r}_l = [\mathbf{r}_l; \mathbf{r}^{-1}_l]$ respectively. For LowFER to encode SimplE, $\mathbf{U}$ becomes a permutation matrix (ignoring the $\frac{1}{2}$ scaling factor), swapping the first $d$-half with the second $d$-half of a given vector in $\mathbb{R}^{2d}$ and $l$-th row in $\mathbf{V}$ is $\hat{\mathbf{r}}_l$, more specifically, with Eq. \ref{eq:final_scoring_func_lowfer}, we set $k=1$, $d_e=2d$, $d_r=n_r$ and $\mathbf{U} \in \mathbb{R}^{2d \times 2d}$ is a block matrix with four partitions such that, $\mathbf{U}_{12}=\mathbf{U}_{21}=\frac{1}{2}\mathbf{I}_d$ and $0$s elsewhere.

\noindent \textbf{\textsc{ComplEx}} \cite{trouillon2016complex} scoring function is defined as:
\begin{multline*}
    \phi_{c}(e_s, r_l, e_o) = \text{Re}(\mathbf{e}_s)^T\text{diag}(\text{Re}(\mathbf{r}_l))\text{Re}(\mathbf{e}_o) \\
    + \text{Im}(\mathbf{e}_s)^T\text{diag}(\text{Re}(\mathbf{r}_l))\text{Im}(\mathbf{e}_o) \\
    + \text{Re}(\mathbf{e}_s)^T\text{diag}(\text{Im}(\mathbf{r}_l))\text{Im}(\mathbf{e}_o) \\
    - \text{Im}(\mathbf{e}_s)^T\text{diag}(\text{Im}(\mathbf{r}_l))\text{Re}(\mathbf{e}_o)
\end{multline*}


\begin{table*}[!tb]
    \centering
    \caption{Link prediction results. Best scores per metric are \textbf{boldfaced} and second best \underline{underlined}.}\label{table:main_results}
    \vspace{0.3cm}
    \resizebox{16cm}{!}{
    \begin{tabular}{clccccccccccccccccccc}
    \toprule 
    & & \multicolumn{4}{c}{WN18RR} & & \multicolumn{4}{c}{FB15k-237} & & \multicolumn{4}{c}{WN18} & & \multicolumn{4}{c}{FB15k} \\ 
    \cmidrule{3-6} \cmidrule{8-11} \cmidrule{13-16} \cmidrule{18-21}
    Linear & Model & MRR & Hits@1 & Hits@3 & Hits@10 & & MRR & Hits@1 & Hits@3 & Hits@10 & & MRR & Hits@1 & Hits@3 & Hits@10 & & MRR & Hits@1 & Hits@3 & Hits@10\\
    \midrule
    \multirow{4}{*}{No} & TransE \cite{bordes2013translating} & $-$ & $-$ & $-$ & $-$ & & $-$ & $-$ & $-$ & $-$ & & $0.454$ & $0.089$ & $0.823$ & $0.934$ & & $0.380$ & $0.231$ & $0.472$ & $0.641$ \\
    & Neural LP \cite{yang2017differentiable} & $-$ & $-$ & $-$ & $-$ & & $0.250$ & $-$ & $-$ & $0.408$ & & $0.940$ & $-$ & $-$ & $0.945$ & & $0.760$ & $-$ & $-$ & $0.837$ \\
    & R-GCN \cite{schlichtkrull2018modeling} & $-$ & $-$ & $-$ & $-$ & & $0.248$ & $0.151$ & $0.264$ & $0.417$ & & $0.819$ & $0.697$ & $0.929$ & $\mathbf{0.964}$ & & $0.696$ & $0.601$ & $0.760$ & $0.842$ \\
    & ConvE \cite{dettmers2018convolutional} & $0.430$ & $0.400$ & $0.440$ & $0.520$ & & $0.325$ & $0.237$ & $0.356$ & $0.501$ & & $0.943$ & $0.935$ & $0.946$ & $0.956$ & & $0.657$ & $0.558$ & $0.723$ & $0.831$ \\
    & TorusE \cite{ebisu2018toruse} & $-$ & $-$ & $-$ & $-$ & & $-$ & $-$ & $-$ & $-$ & & $0.947$ & $0.943$ & $0.950$ & $0.954$ & & $0.733$ & $0.674$ & $0.771$ & $0.832$ \\
    & RotatE \cite{sun2019rotate} & $-$ & $-$ & $-$ & $-$ & & $0.297$ & $0.205$ & $0.328$ & $0.480$ & & $-$ & $-$ & $-$ & $-$ & & $-$ & $-$ & $-$ & $-$ \\
    & HypER \cite{balavzevic2019hypernetwork} & \underline{$0.465$} & \underline{$0.436$} & $0.477$ & $0.522$ & & $0.341$ &$0.252$& $0.376$ & $0.520$ & & \underline{$0.951$} & \underline{$0.947$} & $\mathbf{0.955}$ & \underline{$0.958$} & & $0.790$ & $0.734$ & $0.829$ & $0.885$ \\
    \midrule
    \multirow{6}{*}{Yes} & DistMult \cite{yang2014embedding} & $0.430$ & $0.390$ & $0.440$ & $0.490$ & & $0.241$ & $0.155$ & $0.263$ & $0.419$ & & $0.822$ & $0.728$ & $0.914$ & $0.936$ & & $0.654$ & $0.546$ & $0.733$ & $0.824$ \\
    & HolE \cite{nickel2016holographic} & $-$ & $-$ & $-$ & $-$ & & $-$ & $-$ & $-$ & $-$ & & $0.938$ & $0.930$ & $0.945$ & $0.949$ & & $0.524$ & $0.402$ & $0.613$ & $0.739$ \\
    & ComplEx \cite{trouillon2016complex} & $0.440$ & $0.410$ & $0.460$ & $0.510$ & & $0.247$ & $0.158$ & $0.275$ & $0.428$ & & $0.941$ & $0.936$ & $0.936$ & $0.947$ & & $0.692$ & $0.599$ & $0.759$ & $0.840$ \\
    & ANALOGY \cite{liu2017analogical} & $-$ & $-$ & $-$ & $-$ & & $-$ & $-$ & $-$ & $-$ & & $0.942$ & $0.939$ & $0.944$ & $0.947$ & & $0.725$ & $0.646$ & $0.785$ & $0.854$ \\
    & SimplE \cite{kazemi2018simple} & $-$ & $-$ & $-$ & $-$ & & $-$ & $-$ & $-$ & $-$ & & $0.942$ & $0.939$ & $0.944$ & $0.947$ & & $0.727$ & $0.660$ & $0.773$ & $0.838$ \\
    & TuckER \cite{balazevic2019tucker} & $\mathbf{0.470}$ & $\mathbf{0.443}$ & $\mathbf{0.482}$ & $\mathbf{0.526}$ & & \underline{$0.358$} & $\mathbf{0.266}$ & \underline{$0.394$} & $\mathbf{0.544}$ & & $\mathbf{0.953}$ & $\mathbf{0.949}$ & $\mathbf{0.955}$ & \underline{$0.958$} & & $0.795$ & $0.741$& $0.833$ & $0.892$\\ 
    \cmidrule{2-21}
    & LowFER-1 & $0.454$ & $0.422$ & $0.470$ & $0.515$ & & $0.318$ & $0.233$ & $0.348$ & $0.483$ & & $0.949$ & $0.945$ & $0.951$ & $0.956$ & & $0.720$ & $0.639$ & $0.774$ & $0.859$ \\
    & LowFER-10 & $0.464$ & $0.433$ & $0.477$ & \underline{$0.523$} & & $0.352$ & \underline{$0.261$} & $0.386$ & \underline{$0.533$} & & $0.950$ & $0.946$ & \underline{$0.952$} & \underline{$0.958$} & & \underline{$0.810$} & \underline{$0.760$} & \underline{$0.843$} & \underline{$0.896$} \\
    & LowFER-$k$* & \underline{$0.465$} & $0.434$ & \underline{$0.479$} & $\mathbf{0.526}$ & & $\mathbf{0.359}$ & $\mathbf{0.266}$ & $\mathbf{0.396}$ & $\mathbf{0.544}$ & & $0.950$ & $0.946$ & \underline{$0.952$} & \underline{$0.958$} & & $\mathbf{0.824}$ & $\mathbf{0.782}$ & $\mathbf{0.852}$ & $\mathbf{0.897}$ \\
    \bottomrule
    \end{tabular}
    }
\end{table*}

where Re(.) and Im(.) represents the real and imaginary parts of a complex vector. Consider $\hat{\mathbf{e}}_s=[\text{Re}(\mathbf{e}_s); \text{Im}(\mathbf{e}_s)] \in \mathbb{R}^{2d}$ and $\hat{\mathbf{e}}_o=[\text{Re}(\mathbf{e}_o); \text{Im}(\mathbf{e}_o)] \in \mathbb{R}^{2d}$ then the ComplEx scoring function can be obtained as $\hat{\mathbf{e}}_s^T \mathbf{W}_l \hat{\mathbf{e}}_o$, where $\mathbf{W}_l \in \mathbb{R}^{2d \times 2d}$ represents the $l$-th relation matrix such that its diagonal is $[\text{Re}(\mathbf{r}_l); \text{Re}(\mathbf{r}_l)]$, the $d$ offset diagonal is $\text{Im}(\mathbf{r}_l)$ and $-d$ offset diagonal is $-\text{Im}(\mathbf{r}_l)$. For LowFER to encode ComplEx, similar to SimplE, we will use two permutation matrices to obtain the above four terms. That is, in Eq. \ref{eq:tucker_approx}, we have $k=2$, $d_e=2d$, $d_r=n_r$, $\mathbf{U} \in \mathbb{R}^{2d \times 4d}$ is such that $\mathbf{W}_U^{(1)}$ is a block matrix with $\mathbf{W}_{U_{11}}^{(1)} = \mathbf{W}_{U_{12}}^{(1)} = \mathbf{I}_d$ and $0$ elsewhere. Further, $\mathbf{W}_U^{(2)}$ is also a block matrix with $\mathbf{W}_{U_{21}}^{(2)} = -\mathbf{I}_d$, $\mathbf{W}_{U_{22}}^{(2)} = \mathbf{I}_d$ and $0$ elsewhere. Lastly, $\mathbf{V} \in \mathbb{R}^{n_r \times 4d}$ is such that $\mathbf{W}_V^{(1)}$ row $l$ has $[\text{Re}(\mathbf{r}_l); \text{Im}(\mathbf{r}_l)]$ and $\mathbf{W}_V^{(2)}$ row $l$ has $[\text{Im}(\mathbf{r}_l); \text{Re}(\mathbf{r}_l)]$, i.e., $\mathbf{W}_V^{(2)}=\mathbf{W}_V^{(1)}\mathbf{P}$, where $\mathbf{P} \in \mathbb{R}^{2d \times 2d}$ is the $d$-half swapping permutation matrix. Figure \ref{fig:lowfer_rel_with_bilinear_models} demonstrates LowFER parameters for the \textit{family of bilinear models} under the conditions discussed in this section.


\subsection{Relation to HypER} \label{sec:relation_with_hyper}

HypER \cite{balavzevic2019hypernetwork} is a convolutional model based on \textit{hypernetworks} \cite{ha2016hypernetworks}, where the relation specific 1D filters are generated by the hypernetwork and convolved with the subject entity vector. \citet{balavzevic2019hypernetwork} showed that it can be understood in terms of tensor factorization up to a non-linearity. With a similar argument, we show that LowFER encodes HypER, bringing it closer to the convolutional approaches as well. 

HypER scoring function is defined as \cite{balavzevic2019hypernetwork}:
\begin{equation}
    \phi_h(e_s, r, e_o) = h(\text{vec}(\mathbf{e}_s * \mathbf{F}_r)\mathbf{W})\mathbf{e}_o \label{eq:hyper_score}
\end{equation}
where $\mathbf{F}_r = \text{vec}^{-1}(\mathbf{H}\mathbf{r}) \in \mathbb{R}^{n_f \times l_f}$, $\mathbf{H} \in \mathbb{R}^{n_f l_f \times d_r}$ (hypernetwork), $\mathbf{W} \in \mathbb{R}^{n_f l_m \times d_e}$, $\text{vec}(.)$ transforms $n \times m$ matrix to $nm$-sized vector, $\text{vec}^{-1}(.)$ does the reverse operation, $*$ is the convolution operator, $h(.)$ is ReLU non-linearity and $n_f$, $l_f$ and $l_m=d_e - l_f + 1$ are \textit{number of filters}, \textit{filter length} and \textit{output length} of convolution. The convolution between a filter and the subject entity embedding can be seen as a matrix multiplication, where the filter is converted to a Toeplitz matrix of size ${l_m \times d_e}$. With $n_f$ filters, we can realize a 3D tensor of size $n_f \times l_m \times d_e$. Since the filters are generated by the hypernetwork, we have $d_r$ such 3D tensors, resulting in a 4D tensor of size $n_f \times l_m \times d_e \times d_r$ \cite{balavzevic2019hypernetwork}. Without loss of generality, we can view this 4D tensor as a 3D tensor $\mathcal{F} \in  \mathbb{R}^{n_fl_m \times d_e \times d_r}$. Taking mode-$1$ product as $\mathcal{F} \times_1 \mathbf{W}^T$ returns a final tensor $\mathcal{G} \in \mathbb{R}^{d_e \times d_e \times d_r}$. Thus, HypER operations $\text{vec}(\mathbf{e}_s * \mathbf{F}_r)\mathbf{W}$ simplify to $\mathcal{G} \times_3 \mathbf{r} \times_2 \mathbf{e}_s$. At $k=d_e$, with $\mathbf{U} \in \mathbb{R}^{d_e \times d_e^2}$ as block identity matrices (same as in LowFER's relation to RESCAL) and $\mathbf{V} \in \mathbb{R}^{d_r \times de^2}$ set to $\mathbf{G}^T$ ($\mathcal{G}$ viewed as a matrix of size $d_e^2 \times d_r$ and transposed), LowFER's score in Eq. \ref{eq:final_scoring_func_lowfer} represents HypER, up to the non-linearity.


\section{Experiments and Results}

We conducted the experiments on four benchmark datasets: WN18 \cite{bordes2013translating}, WN18RR \cite{dettmers2018convolutional}, FB15k \cite{bordes2013translating} and FB15k-237 \cite{toutanova2015representing} (see Appendix \ref{appendix:B} for the details, including best hyperparameters and additional experiments).


\subsection{Link Prediction} \label{sec:link_prediction_results}

Table \ref{table:main_results} shows our main results, where LowFER-$1$, LowFER-$10$ and LowFER-$k$* represent our model for $k=1$, $k=10$ and $k=\text{best}$.  We choose LowFER-$1$ and LowFER-$10$ as baselines. Overall, LowFER reaches competitive performance on all the datasets with state-of-the-art results on FB15k and FB15k-237. On WN18 and WN18RR, TuckER is marginally better than LowFER.

LowFER performs well at low-ranks with significantly less number of parameters compared to other linear models (Table \ref{table:linear_model_parameters}). At $k=1$, it performs better than or on par with both non-linear and linear models (including ComplEx and SimplE) except HypER and TuckER. For FB15k-237, LowFER-$1$ ($\sim$3M parameters) outperforms R-GCN, RotatE, DistMult and ComplEx by an average of 5.9\% on MRR, and it additionally outperforms convolutional models (ConvE, HypER) at $k=10$ with only $+$0.8M parameters. On FB15k, the best reported TuckER model is improved upon, with absolute +$1.9$\% increase on toughest Hits@1 metric. This already achieves state-of-the-art with almost half the parameters, $\sim$5.5M in contrast to TuckER's $\sim$11.3M. On WN18RR and WN18, LowFER-$1$ outperforms all the models excluding TuckER and HypER. With LowFER-$k$*, we marginally reach state-of-the-art performance on WN18RR and FB15k-237. On FB15k, we reach new state-of-the-art for $\sim$9.51M parameters with +$2.9$\% and +$4.1$\% improvement on MRR and Hits@1. 

The empirical gains can be attributed to LowFER's ability to perform \textit{good} fusion between entities and relations while avoiding overfitting through low-rank matrices remaining parameter efficient, with strong performance even at extreme low-ranks. Further, like TuckER, it allows for parameter sharing through the $\mathbf{U}$ and $\mathbf{V}$ matrices, unlike ComplEx and SimplE which rely only on embedding matrices. 


\begin{table}[!tb]
    \centering
    \caption{Comparison between the number of parameters in millions (M) of strong linear models. For LowFER-$k$*, the $k$ values are $10$, $100$, $30$ and $50$ for WN18, FB15k-237, WN18RR and FB15k respectively.}\label{table:linear_model_parameters}
    \vspace{0.3cm}
    \resizebox{7cm}{!}{
    \begin{tabular}{lcccc}
    \toprule
    Model & WN18 & FB15k-237 & WN18RR & FB15k \\
    \midrule
    ComplEx & $16.4$ & $6.0$ & $16.4$ & $6.5$ \\
    SimplE  & $16.4$ & - & $16.4$ & $6.5$ \\
    TuckER  & $9.4$ & $11.0$ & $9.4$ & $11.3$ \\
    LowFER-1 & $8.2$ & $3.0$ & $8.2$ & $4.6$ \\
    LowFER-10 & $8.6$ & $3.8$ & $8.6$ & $5.5$ \\
    LowFER-$k$* & $8.6$ & $11.3$ & $9.6$ & $9.5$ \\ 
    \bottomrule
    \end{tabular}
    }
\end{table}


\begin{figure}[!tbp]
    \includegraphics[width=0.85\linewidth]{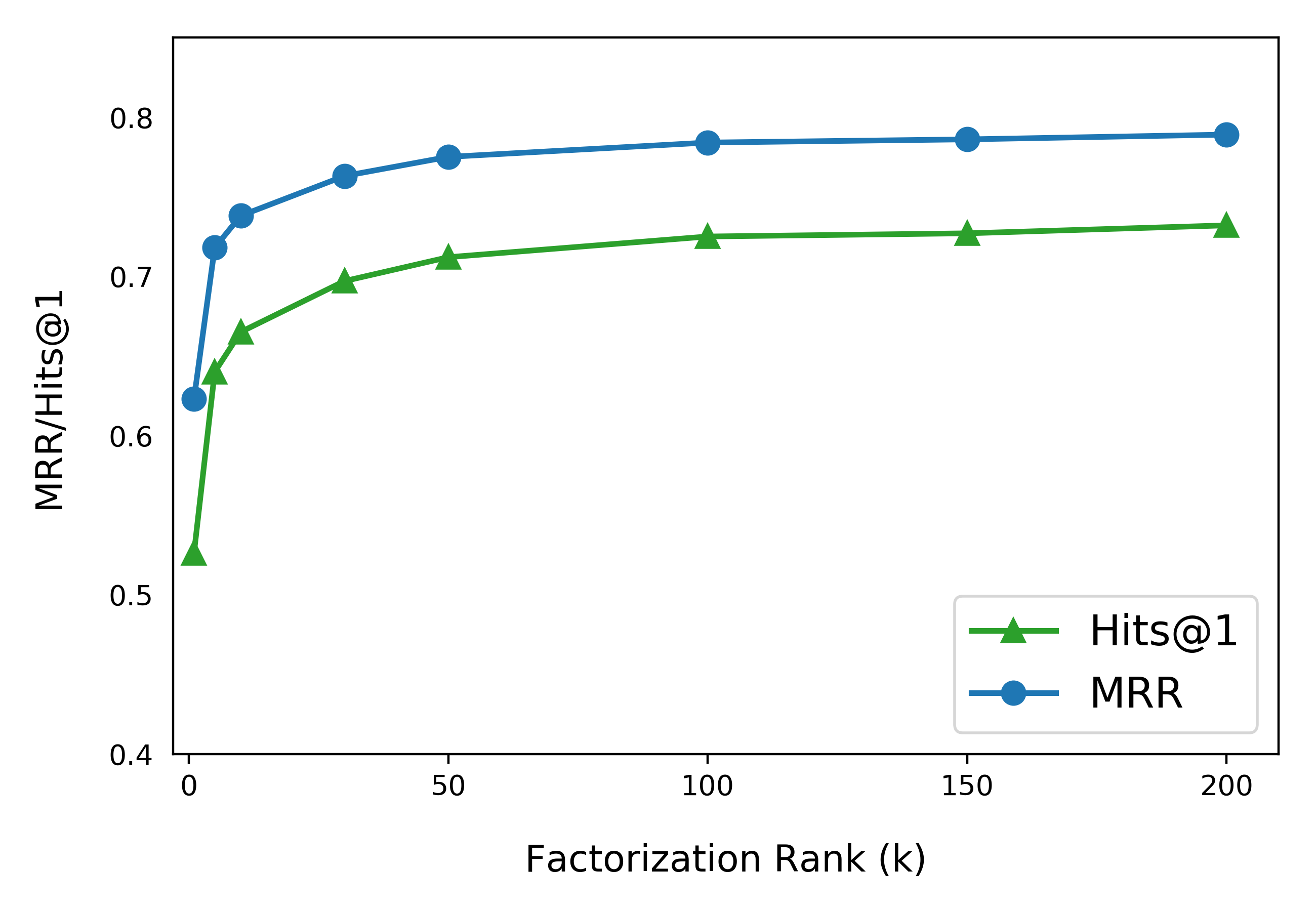}
    \centering
    \caption{Influence of increasing the factorization rank on MRR and Hits@1 scores for FB15k.} \label{fig:k_influence}
\end{figure}


\subsection{Effect of Factorization Rank}

From link prediction results, we observe that rank plays an important role depending on the entities-to-relations ratio in the dataset. For $d_e=200$ and $d_r=30$, we vary $k$ from $\{1, 5, 10, 30, 50, 100, 150, 200\}$ on FB15k and plot the MRR and Hits@1 scores (Figure \ref{fig:k_influence}). From $k=1$ to $k=5$, the MRR score increases from $0.62$ to $0.72$ and Hits@1 increases from $0.53$ to $0.64$. For higher ranks (after $50$), the change is minimal. Empirically, the effect of $k$ diminishes as the number of the entities per relation becomes larger, e.g., it is $\sim3722$ for WN18RR in contrast to $\sim11$ for FB15k. We suspect that this could be due to the fact that as $n_e \gg d_e$, most of the knowledge is learned through embedding matrices rather than the model parameters $\mathbf{U}$ and $\mathbf{V}$. To test this, we took a trained LowFER model, on WN18 dataset, and added zero mean Gaussian noise with variance in $\{1.0, 1.25, 1.5, 1.75, 2.0\}$ to $\mathbf{U}$ and $\mathbf{V}$ and evaluated on the test set. The MRR score changed from $0.95$ to $\{0.92, 0.84, 0.65, 0.42, 0.24\}$ for each level of noise. This shows that in cases as such, the embeddings have potential to capture more knowledge than the shared parameters.

Empirically, we found when $d_e=d_r$, taking $k=d_e/2$ performs nearly the same as TuckER \cite{balazevic2019tucker}. This can be observed in LowFER-$k$* for FB15k-237 ($d_e=d_r=200$, $k=100$), where our results are almost indistinguishable from TuckER's. This can be expected as the number of parameters in both models are almost the same ($\sim$11M). It should be noted that in practice when we train LowFER, we initialize with two i.i.d matrices, which are not shared, compared to TuckER's core tensor (Eq. \ref{eq:tucker_scoring_func}), allowing us to reach almost the same performance despite less parameter sharing.


\begin{figure}[!tbp]
    \includegraphics[width=0.85\linewidth]{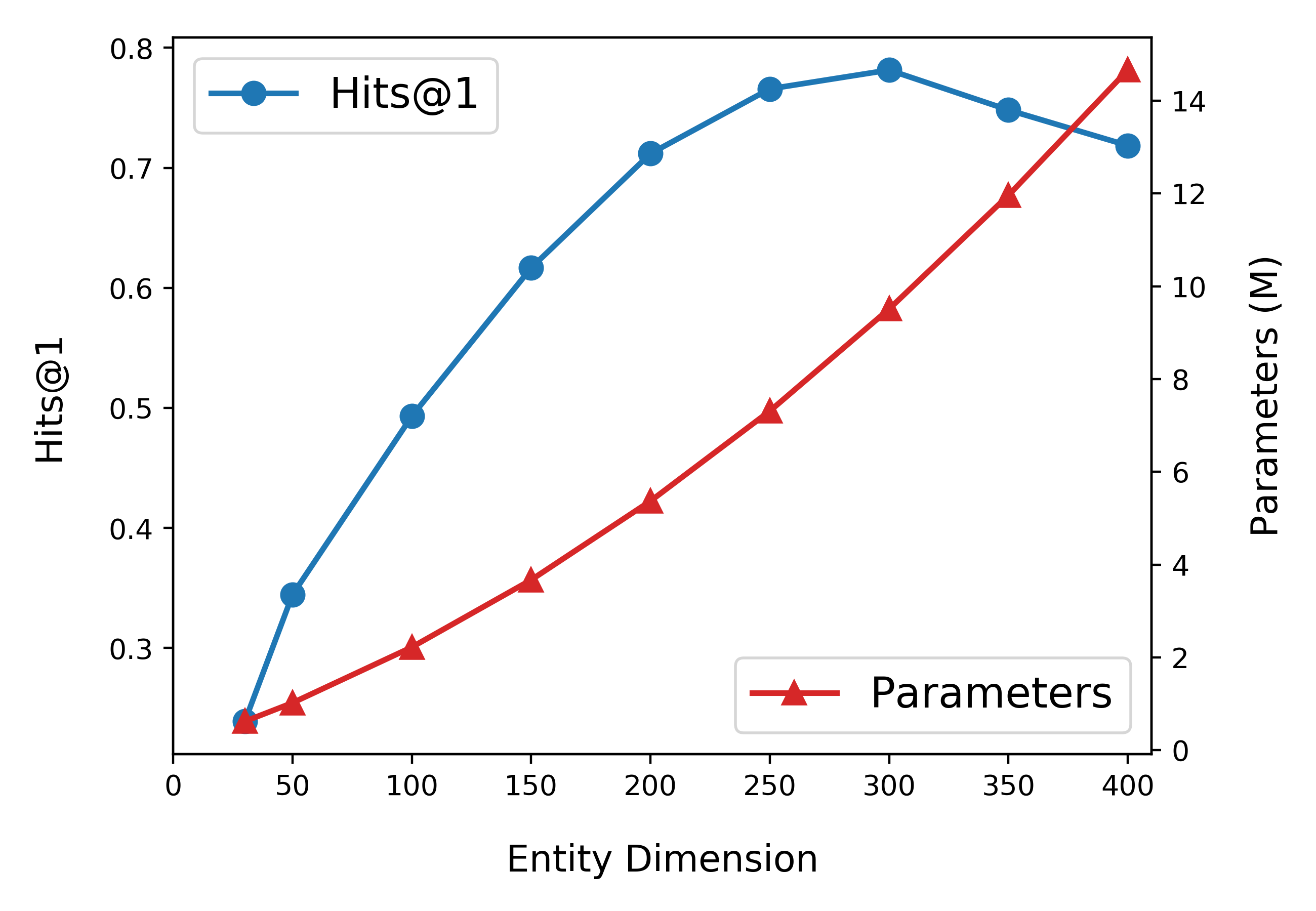}
    \centering
    \caption{Influence of changing the entity embedding dimension $d_e$ on Hits@1 metric and growth of parameters in million (M).}
    \label{fig:de_effect1}
\end{figure}


\begin{table}[!tbp]
    \centering
    \caption{Link prediction results on FB15k with $d_e=d_r=200$.}\label{tab:fb15k_200_200_exps}
    \vspace{0.3cm}
    \resizebox{7.5cm}{!}{
    \begin{tabular}{rrcccc}
    \toprule 
    $k$ & Params (M) & MRR & Hits@1 & Hits@3 & Hits@10\\
    \midrule
    $1$ & $3.60$ & $0.634$ & $0.538$ & $0.695$ & $0.803$ \\
    $5$ & $3.92$ & $0.720$ & $0.641$ & $0.776$ & $0.860$ \\
    $10$ & $4.33$ & $0.742$ & $0.667$ & $0.790$ & $0.871$ \\
    $30$ & $5.93$ & $0.774$ & $0.709$ & $0.817$ & $0.885$ \\
    $50$ & $7.53$ & $0.776$ & $0.713$ & $0.818$ & $0.886$ \\
    $100$ & $11.53$ & $0.779$ & $0.717$ & $0.821$ & $0.887$ \\
    \bottomrule
    \end{tabular}
    }
 \end{table}


\subsection{Effect of Embedding Dimension} \label{sec:embedding_exps}

The size of entity embedding dimension $d_e$ accounts for the significant number of parameters in LowFER, growing linearly with number of entities $n_e$. To study the effect, we trained our models on FB15k, with $d_r=30$, $k=50$ constant, and varying $de$ in $\{30, 50, 100, 150, 200, 250, 300, 350, 400\}$. As can be seen in Figure \ref{fig:de_effect1}, increasing the entity embedding dimension significantly increases the Hits@1 metric, for almost linear growth in number of parameters. However, it only improves till $300$ and starts overfitting afterwards.

In \citet{balazevic2019tucker}, authors reported $d_e=d_r=200$ as best choice of dimensions for TuckER on FB15k, however, we found using $de=300$ and $d_r=30$ better with lesser number of parameters for LowFER. For fair comparison, we also provide the results for $d_e=d_r=200$ for $k$ in $\{1, 5, 10, 30, 50, 100\}$ in Table \ref{tab:fb15k_200_200_exps}. As $k$ is increased, we see an improvement over all the metrics. At $k=100$, where we expected LowFER to match TuckER's performance (MRR=$0.795$, Hits@1=$0.741$, $\sim$11 million parameters), it was lower ($-1.6$\% on MRR and $-2.4$\% on Hits@1). In comparison, our model with $d_e=300$, $d_r=30$ and $k=10$ with $\sim$5.6 million parameters only, gives better results than this setting and TuckER. Therefore, at $d_e=d_r=200$, our model is most likely overfitting. 

As noted above that it could be that LowFER is overfitting therefore, we did coarse grid search over relation embedding dimension in $\{30, 50, 100, 150, 200\}$ and $k$ in $\{1, 5, 10, 30, 50, 100, 150, 200\}$ while keeping $d_e=200$ fixed. We found $d_r=50$ at $k=150$ reaches almost the same performance as TuckER with $\sim$10.6M parameters compared to TuckER's $\sim$11.3M parameters. We also experimented with $l_2$-regularization (Reg) and noted minor improvements, with regularization strength $0.0005$. Table \ref{tab:fb15k_200_50_150} summarizes these results. Note that all the experiments reported in main results (Table \ref{table:main_results}) were without any regularization. In general, we only noticed slight improvements in FB15k with $l_2$-regularization.


\begin{table}[!tbp]
    \centering
    \caption{Link prediction results on FB15k with $d_e=200$, $d_r=50$, $k=150$ and $l_2$-regularization $0.0005$.}\label{tab:fb15k_200_50_150}
    \vspace{0.3cm}
    \resizebox{7.5cm}{!}{
    \begin{tabular}{llcccc}
    \toprule 
    Model & Params (M) & MRR & Hits@1 & Hits@3 & Hits@10\\
    \midrule
    TuckER & $11.3$ & $0.795$ & $0.741$& $0.833$ & $0.892$ \\
    LowFER-$k$* & $10.6$ & $0.795$ & $0.739$ & $0.831$ & $0.891$ \\
    LowFER-$k$* $+$ Reg & $10.6$ & $0.802$ & $0.749$ & $0.837$ & $0.892$ \\
    \bottomrule
    \end{tabular}
    }
 \end{table}


\subsection{Analysis of Relation Results} \label{sec:per_rel_analysis}

Link prediction models that can discover relation types automatically without prior knowledge indicate \textit{better} generalization. As shown, and discussed in section \ref{sec:theory}, LowFER, among other models (Table \ref{table:bounds_summary}), can learn to capture all relation types without additional constraints. However, in practice, these bounds are loose and require very large dimensions, raising an inspection into their performance on different relation types. In \citet{kazemi2018simple}, it was identified that WN18 contains redundant relations, i.e., $\forall e_i, e_j \in \mathcal{E}: (e_i, r_1, e_j) \in \mathcal{T} \Leftrightarrow (e_j, r_2, e_i) \in \mathcal{T}$, such as $<$\textit{hyponym}, \textit{hypernym}$>$, $<$\textit{meronym}, \textit{holonym}$>$ etc. To alleviate this, \citet{dettmers2018convolutional} proposed WN18RR with such relations removed, since knowledge about one can help infer the knowledge about the other. Table \ref{table:rel_specific_results} shows the per relation results of LowFER and TuckER on WN18 and WN18RR. We see that performance drops for $7$ relations, with an average performance decrease of $-70.6$\% and $-69.3$\% for LowFER and TuckER respectively (with highest decrease on \textit{member\_of\_domain\_usage} for both). For symmetric relations (such as \textit{derivationally\_related\_form}), the performance is approximately the same where we observe severe limitation to model asymmetry. We believe this is because LowFER (also TuckER) is constraint-free and adding certain constraints based on background knowledge is \textit{necessary} to improve the model's accuracy. SimplE is the only \textit{fully expressive} model that has formally shown to address these limitations (cf. Proposition 3, 4 and 5 in \citet{kazemi2018simple}). Since LowFER subsumes SimplE therefore, such rules can be studied for extending LowFER to incorporate the background knowledge. 


\begin{table}[!t]
    \centering
    \caption{Relation specific test set results on WN18 and WN18RR with LowFER-$k$* and best reported TuckER model \cite{balazevic2019tucker}.} \label{table:rel_specific_results}
    \vspace{0.3cm}
    \resizebox{8cm}{!}{
    \begin{tabular}{lccccc}
    \toprule 
    & \multicolumn{2}{c}{WN18}&&\multicolumn{2}{c}{WN18RR} \\ 
    \cmidrule{2-3} \cmidrule{5-6}
    & LowFER & TuckER & & LowFER & TuckER \\
    \midrule
    also\_see & $0.638$ & $0.630$ & & $0.627$ & $0.614$ \\
    derivationally\_related\_form & $0.954$ & $0.956$ & & $0.957$ & $0.957$ \\
    has\_part & $0.944$ & $0.945$ & & $0.138$ & $0.129$ \\
    hypernym & $0.961$ & $0.962$ & & $0.189$ & $0.189$ \\
    instance\_hypernym & $0.986$ & $0.982$ & & $0.576$ & $0.591$ \\ 
    member\_meronym & $0.930$ & $0.927$ & & $0.155$ & $0.131$ \\
    member\_of\_domain\_region & $0.885$ & $0.885$ & & $0.060$ & $0.083$ \\
    member\_of\_domain\_usage & $0.917$ & $0.917$ & & $0.025$ & $0.096$ \\
    similar\_to & $1.0$ & $1.0$ & & $1.0$ & $1.0$ \\
    synset\_domain\_topic\_of & $0.956$ & $0.952$ & & $0.494$ & $0.499$ \\
    verb\_group & $0.974$ & $0.974$ & & $0.974$ & $0.974$ \\
    \bottomrule
    \end{tabular}
    }
 \end{table}


\section{Conclusion}

This work proposes a simple and parameter efficient \textit{fully expressive} linear model that is theoretically well sound and performs on par or state-of-the-art in practice. We showed that LowFER generalizes to other linear models in KGC, providing a unified theoretical view. It offers a strong baseline to the deep learning based models and raises further interest into the study of linear models. We also highlighted some limitations with respect to gains on harder relations, which still pose a challenge. We conclude that the constraint-free and parameter efficient linear models, which allow for parameter sharing, are better from a modeling perspective, but are still similarly limited in learning difficult relations. Therefore, studying the trade-off between parameters sharing and constraints becomes an important future work.


\section*{Acknowledgements}

The authors would like to thank the anonymous reviewers for helpful feedback and gratefully acknowledge the use of code released by \citet{balazevic2019tucker}. The work was partially funded by the European Union's Horizon 2020 research and innovation programme under grant agreement No. 777107 through the project Precise4Q and by the German Federal Ministry of Education and Research (BMBF) through the project DEEPLEE (01IW17001).

\nocite{kingma2014adam,ioffe2015batch,szegedy2016rethinking,pereyra2017regularizing,hayashi2017equivalence,mahdisoltani2013yago3,miettinen2011boolean,de2000multilinear}

\bibliography{references}

\begin{thebibliography}{52}
\providecommand{\natexlab}[1]{#1}
\providecommand{\url}[1]{\texttt{#1}}
\expandafter\ifx\csname urlstyle\endcsname\relax
  \providecommand{\doi}[1]{doi: #1}\else
  \providecommand{\doi}{doi: \begingroup \urlstyle{rm}\Url}\fi

\bibitem[Bala{\v{z}}evi{\'c} et~al.(2019{\natexlab{a}})Bala{\v{z}}evi{\'c},
  Allen, and Hospedales]{balazevic2019tucker}
Bala{\v{z}}evi{\'c}, I., Allen, C., and Hospedales, T.
\newblock {TuckER: Tensor Factorization for Knowledge Graph Completion}.
\newblock In \emph{Proceedings of the 2019 Conference on Empirical Methods in
  Natural Language Processing and the 9th International Joint Conference on
  Natural Language Processing (EMNLP-IJCNLP)}, pp.\  5188--5197,
  2019{\natexlab{a}}.

\bibitem[Bala{\v{z}}evi{\'c} et~al.(2019{\natexlab{b}})Bala{\v{z}}evi{\'c},
  Allen, and Hospedales]{balavzevic2019hypernetwork}
Bala{\v{z}}evi{\'c}, I., Allen, C., and Hospedales, T.~M.
\newblock {Hypernetwork knowledge graph embeddings}.
\newblock In \emph{International Conference on Artificial Neural Networks},
  pp.\  553--565. Springer, 2019{\natexlab{b}}.

\bibitem[Ben-Younes et~al.(2017)Ben-Younes, Cadene, Cord, and
  Thome]{ben2017mutan}
Ben-Younes, H., Cadene, R., Cord, M., and Thome, N.
\newblock {MUTAN: Multimodal Tucker Fusion for Visual Question Answering}.
\newblock In \emph{International Conference on Computer Vision}, 2017.

\bibitem[Bordes et~al.(2013)Bordes, Usunier, Garcia-Duran, Weston, and
  Yakhnenko]{bordes2013translating}
Bordes, A., Usunier, N., Garcia-Duran, A., Weston, J., and Yakhnenko, O.
\newblock {Translating Embeddings for Modeling Multi-relational Data}.
\newblock In \emph{Advances in Neural Information Processing Systems}, 2013.

\bibitem[Carroll \& Chang(1970)Carroll and Chang]{carroll1970analysis}
Carroll, J.~D. and Chang, J.-J.
\newblock Analysis of individual differences in multidimensional scaling via an
  n-way generalization of “eckart-young” decomposition.
\newblock \emph{Psychometrika}, 35\penalty0 (3):\penalty0 283--319, 1970.

\bibitem[Charikar et~al.(2004)Charikar, Chen, and
  Farach-Colton]{charikar2004finding}
Charikar, M., Chen, K., and Farach-Colton, M.
\newblock Finding frequent items in data streams.
\newblock \emph{Theoretical Computer Science}, 312\penalty0 (1):\penalty0
  3--15, 2004.

\bibitem[Cichocki et~al.(2016)Cichocki, Lee, Oseledets, Phan, Zhao, Mandic,
  et~al.]{cichocki2016tensor}
Cichocki, A., Lee, N., Oseledets, I., Phan, A.-H., Zhao, Q., Mandic, D.~P.,
  et~al.
\newblock {Tensor networks for dimensionality reduction and large-scale
  optimization: {P}art 1 low-rank tensor decompositions}.
\newblock \emph{Foundations and Trends{\textregistered} in Machine Learning},
  9\penalty0 (4-5):\penalty0 249--429, 2016.

\bibitem[Cybenko(1989)]{cybenko1989approximation}
Cybenko, G.
\newblock Approximation by superpositions of a sigmoidal function.
\newblock \emph{Mathematics of control, signals and systems}, 2\penalty0
  (4):\penalty0 303--314, 1989.

\bibitem[Das et~al.(2017)Das, Dhuliawala, Zaheer, Vilnis, Durugkar,
  Krishnamurthy, Smola, and McCallum]{das2017go}
Das, R., Dhuliawala, S., Zaheer, M., Vilnis, L., Durugkar, I., Krishnamurthy,
  A., Smola, A., and McCallum, A.
\newblock Go for a walk and arrive at the answer: Reasoning over paths in
  knowledge bases using reinforcement learning.
\newblock \emph{arXiv preprint arXiv:1711.05851}, 2017.

\bibitem[De~Lathauwer et~al.(2000)De~Lathauwer, De~Moor, and
  Vandewalle]{de2000multilinear}
De~Lathauwer, L., De~Moor, B., and Vandewalle, J.
\newblock A multilinear singular value decomposition.
\newblock \emph{SIAM journal on Matrix Analysis and Applications}, 21\penalty0
  (4):\penalty0 1253--1278, 2000.

\bibitem[Dettmers et~al.(2018)Dettmers, Minervini, Stenetorp, and
  Riedel]{dettmers2018convolutional}
Dettmers, T., Minervini, P., Stenetorp, P., and Riedel, S.
\newblock {Convolutional 2D Knowledge Graph Embeddings}.
\newblock In \emph{Association for the Advancement of Artificial Intelligence},
  2018.

\bibitem[Ebisu \& Ichise(2018)Ebisu and Ichise]{ebisu2018toruse}
Ebisu, T. and Ichise, R.
\newblock {Toruse: Knowledge graph embedding on a lie group}.
\newblock In \emph{Thirty-Second AAAI Conference on Artificial Intelligence},
  2018.

\bibitem[Feng et~al.(2016)Feng, Huang, Wang, Zhou, Hao, and
  Zhu]{feng2016knowledge}
Feng, J., Huang, M., Wang, M., Zhou, M., Hao, Y., and Zhu, X.
\newblock {Knowledge graph embedding by flexible translation}.
\newblock In \emph{Fifteenth International Conference on the Principles of
  Knowledge Representation and Reasoning}, 2016.

\bibitem[Fukui et~al.(2016)Fukui, Park, Yang, Rohrbach, Darrell, and
  Rohrbach]{fukui2016multimodal}
Fukui, A., Park, D.~H., Yang, D., Rohrbach, A., Darrell, T., and Rohrbach, M.
\newblock {Multimodal Compact Bilinear Pooling for Visual Question Answering
  and Visual Grounding}.
\newblock In \emph{Conference on Empirical Methods in Natural Language
  Processing}, pp.\  457--468. ACL, 2016.

\bibitem[Gao et~al.(2016)Gao, Beijbom, Zhang, and Darrell]{gao2016compact}
Gao, Y., Beijbom, O., Zhang, N., and Darrell, T.
\newblock Compact bilinear pooling.
\newblock In \emph{Proceedings of the IEEE conference on computer vision and
  pattern recognition}, pp.\  317--326, 2016.

\bibitem[Ha et~al.(2017)Ha, Dai, and Le]{ha2016hypernetworks}
Ha, D., Dai, A., and Le, Q.~V.
\newblock {Hypernetworks}.
\newblock In \emph{International Conference on Learning Representations}, 2017.

\bibitem[Harshman(1978)]{harshman1978models}
Harshman, R.~A.
\newblock {Models for analysis of asymmetrical relationships among N objects or
  stimuli}.
\newblock In \emph{First Joint Meeting of the Psychometric Society and the
  Society of Mathematical Psychology, Hamilton, Ontario, 1978}, 1978.

\bibitem[Harshman \& Lundy(1994)Harshman and Lundy]{harshman1994parafac}
Harshman, R.~A. and Lundy, M.~E.
\newblock {PARAFAC: Parallel factor analysis}.
\newblock \emph{Computational Statistics \& Data Analysis}, 18\penalty0
  (1):\penalty0 39--72, 1994.

\bibitem[Hayashi \& Shimbo(2017)Hayashi and Shimbo]{hayashi2017equivalence}
Hayashi, K. and Shimbo, M.
\newblock On the equivalence of holographic and complex embeddings for link
  prediction.
\newblock \emph{arXiv preprint arXiv:1702.05563}, 2017.

\bibitem[Hitchcock(1927)]{hitchcock1927expression}
Hitchcock, F.~L.
\newblock {The expression of a tensor or a polyadic as a sum of products}.
\newblock \emph{Journal of Mathematics and Physics}, 6\penalty0 (1-4):\penalty0
  164--189, 1927.

\bibitem[Hornik(1991)]{hornik1991approximation}
Hornik, K.
\newblock Approximation capabilities of multilayer feedforward networks.
\newblock \emph{Neural networks}, 4\penalty0 (2):\penalty0 251--257, 1991.

\bibitem[Ioffe \& Szegedy(2015)Ioffe and Szegedy]{ioffe2015batch}
Ioffe, S. and Szegedy, C.
\newblock {Batch Normalization: Accelerating Deep Network Training by Reducing
  Internal Covariate Shift}.
\newblock In \emph{International Conference on Machine Learning}, 2015.

\bibitem[Ji et~al.(2015)Ji, He, Xu, Liu, and Zhao]{ji2015knowledge}
Ji, G., He, S., Xu, L., Liu, K., and Zhao, J.
\newblock Knowledge graph embedding via dynamic mapping matrix.
\newblock In \emph{Proceedings of the 53rd Annual Meeting of the Association
  for Computational Linguistics and the 7th International Joint Conference on
  Natural Language Processing (Volume 1: Long Papers)}, pp.\  687--696, 2015.

\bibitem[Kazemi \& Poole(2018)Kazemi and Poole]{kazemi2018simple}
Kazemi, S.~M. and Poole, D.
\newblock {SimplE Embedding for Link Prediction in Knowledge Graphs}.
\newblock In \emph{Advances in Neural Information Processing Systems}, 2018.

\bibitem[Kim et~al.(2016)Kim, On, Lim, Kim, Ha, and Zhang]{kim2016hadamard}
Kim, J.-H., On, K.-W., Lim, W., Kim, J., Ha, J.-W., and Zhang, B.-T.
\newblock {Hadamard product for low-rank bilinear pooling}.
\newblock \emph{arXiv preprint arXiv:1610.04325}, 2016.

\bibitem[Kingma \& Ba(2015)Kingma and Ba]{kingma2014adam}
Kingma, D.~P. and Ba, J.
\newblock {Adam: A Method for Stochastic Optimization}.
\newblock In \emph{International Conference on Learning Representations}, 2015.

\bibitem[Lacroix et~al.(2018)Lacroix, Usunier, and
  Obozinski]{lacroix2018canonical}
Lacroix, T., Usunier, N., and Obozinski, G.
\newblock {Canonical Tensor Decomposition for Knowledge Base Completion}.
\newblock In \emph{International Conference on Machine Learning}, 2018.

\bibitem[Li et~al.(2017)Li, Wang, Liu, and Hou]{li2017factorized}
Li, Y., Wang, N., Liu, J., and Hou, X.
\newblock {Factorized bilinear models for image recognition}.
\newblock In \emph{Proceedings of the IEEE International Conference on Computer
  Vision}, pp.\  2079--2087, 2017.

\bibitem[Lin et~al.(2015)Lin, Liu, Luan, Sun, Rao, and Liu]{lin2015modeling}
Lin, Y., Liu, Z., Luan, H., Sun, M., Rao, S., and Liu, S.
\newblock {Modeling relation paths for representation learning of knowledge
  bases}.
\newblock \emph{arXiv preprint arXiv:1506.00379}, 2015.

\bibitem[Liu et~al.(2017)Liu, Wu, and Yang]{liu2017analogical}
Liu, H., Wu, Y., and Yang, Y.
\newblock {Analogical inference for multi-relational embeddings}.
\newblock In \emph{Proceedings of the 34th International Conference on Machine
  Learning-Volume 70}, pp.\  2168--2178. JMLR. org, 2017.

\bibitem[Liu et~al.(2018)Liu, Shen, Lakshminarasimhan, Liang, Zadeh, and
  Morency]{liu2018efficient}
Liu, Z., Shen, Y., Lakshminarasimhan, V.~B., Liang, P.~P., Zadeh, A., and
  Morency, L.-P.
\newblock {Efficient low-rank multimodal fusion with modality-specific
  factors}.
\newblock \emph{arXiv preprint arXiv:1806.00064}, 2018.

\bibitem[Mahdisoltani et~al.(2013)Mahdisoltani, Biega, and
  Suchanek]{mahdisoltani2013yago3}
Mahdisoltani, F., Biega, J., and Suchanek, F.~M.
\newblock Yago3: A knowledge base from multilingual wikipedias.
\newblock 2013.

\bibitem[Miettinen(2011)]{miettinen2011boolean}
Miettinen, P.
\newblock Boolean tensor factorizations.
\newblock In \emph{2011 IEEE 11th International Conference on Data Mining},
  pp.\  447--456. IEEE, 2011.

\bibitem[Nguyen et~al.(2016)Nguyen, Sirts, Qu, and Johnson]{nguyen2016stranse}
Nguyen, D.~Q., Sirts, K., Qu, L., and Johnson, M.
\newblock {Stranse: a novel embedding model of entities and relationships in
  knowledge bases}.
\newblock \emph{arXiv preprint arXiv:1606.08140}, 2016.

\bibitem[Nickel et~al.(2011)Nickel, Tresp, and Kriegel]{nickel2011three}
Nickel, M., Tresp, V., and Kriegel, H.-P.
\newblock {A Three-Way Model for Collective Learning on Multi-Relational Data}.
\newblock In \emph{International Conference on Machine Learning}, 2011.

\bibitem[Nickel et~al.(2016)Nickel, Rosasco, and Poggio]{nickel2016holographic}
Nickel, M., Rosasco, L., and Poggio, T.
\newblock {Holographic embeddings of knowledge graphs}.
\newblock In \emph{Thirtieth AAAI conference on artificial intelligence}, 2016.

\bibitem[Pereyra et~al.(2017)Pereyra, Tucker, Chorowski, Kaiser, and
  Hinton]{pereyra2017regularizing}
Pereyra, G., Tucker, G., Chorowski, J., Kaiser, {\L}., and Hinton, G.
\newblock Regularizing neural networks by penalizing confident output
  distributions.
\newblock \emph{arXiv preprint arXiv:1701.06548}, 2017.

\bibitem[Pham \& Pagh(2013)Pham and Pagh]{pham2013fast}
Pham, N. and Pagh, R.
\newblock Fast and scalable polynomial kernels via explicit feature maps.
\newblock In \emph{Proceedings of the 19th ACM SIGKDD international conference
  on Knowledge discovery and data mining}, pp.\  239--247, 2013.

\bibitem[Schlichtkrull et~al.(2018)Schlichtkrull, Kipf, Bloem, Van Den~Berg,
  Titov, and Welling]{schlichtkrull2018modeling}
Schlichtkrull, M., Kipf, T.~N., Bloem, P., Van Den~Berg, R., Titov, I., and
  Welling, M.
\newblock {Modeling relational data with graph convolutional networks}.
\newblock In \emph{European Semantic Web Conference}, pp.\  593--607. Springer,
  2018.

\bibitem[Shen et~al.(2018)Shen, Chen, Huang, Guo, and Gao]{shen2018m}
Shen, Y., Chen, J., Huang, P.-S., Guo, Y., and Gao, J.
\newblock M-walk: Learning to walk over graphs using monte carlo tree search.
\newblock In \emph{Advances in Neural Information Processing Systems}, pp.\
  6786--6797, 2018.

\bibitem[Sun et~al.(2019)Sun, Deng, Nie, and Tang]{sun2019rotate}
Sun, Z., Deng, Z.-H., Nie, J.-Y., and Tang, J.
\newblock {Rotate: Knowledge graph embedding by relational rotation in complex
  space}.
\newblock \emph{arXiv preprint arXiv:1902.10197}, 2019.

\bibitem[Szegedy et~al.(2016)Szegedy, Vanhoucke, Ioffe, Shlens, and
  Wojna]{szegedy2016rethinking}
Szegedy, C., Vanhoucke, V., Ioffe, S., Shlens, J., and Wojna, Z.
\newblock Rethinking the inception architecture for computer vision.
\newblock In \emph{Proceedings of the IEEE conference on computer vision and
  pattern recognition}, pp.\  2818--2826, 2016.

\bibitem[Toutanova et~al.(2015)Toutanova, Chen, Pantel, Poon, Choudhury, and
  Gamon]{toutanova2015representing}
Toutanova, K., Chen, D., Pantel, P., Poon, H., Choudhury, P., and Gamon, M.
\newblock {Representing Text for Joint Embedding of Text and Knowledge Bases}.
\newblock In \emph{Empirical Methods in Natural Language Processing}, 2015.

\bibitem[Trouillon \& Nickel(2017)Trouillon and Nickel]{trouillon2017complex}
Trouillon, T. and Nickel, M.
\newblock {Complex and holographic embeddings of knowledge graphs: a
  comparison}.
\newblock \emph{arXiv preprint arXiv:1707.01475}, 2017.

\bibitem[Trouillon et~al.(2016)Trouillon, Welbl, Riedel, Gaussier, and
  Bouchard]{trouillon2016complex}
Trouillon, T., Welbl, J., Riedel, S., Gaussier, {\'E}., and Bouchard, G.
\newblock {Complex Embeddings for Simple Link Prediction}.
\newblock In \emph{International Conference on Machine Learning}, 2016.

\bibitem[Trouillon et~al.(2017)Trouillon, Dance, Gaussier, Welbl, Riedel, and
  Bouchard]{trouillon2017knowledge}
Trouillon, T., Dance, C.~R., Gaussier, {\'E}., Welbl, J., Riedel, S., and
  Bouchard, G.
\newblock Knowledge graph completion via complex tensor factorization.
\newblock \emph{The Journal of Machine Learning Research}, 18\penalty0
  (1):\penalty0 4735--4772, 2017.

\bibitem[Tucker(1966)]{tucker1966some}
Tucker, L.~R.
\newblock {Some mathematical notes on three-mode factor analysis}.
\newblock \emph{Psychometrika}, 31\penalty0 (3):\penalty0 279--311, 1966.

\bibitem[Wang et~al.(2018)Wang, Gemulla, and Li]{wang2018multi}
Wang, Y., Gemulla, R., and Li, H.
\newblock {On multi-relational link prediction with bilinear models}.
\newblock In \emph{Thirty-Second AAAI Conference on Artificial Intelligence},
  2018.

\bibitem[Wang et~al.(2014)Wang, Zhang, Feng, and Chen]{wang2014knowledge}
Wang, Z., Zhang, J., Feng, J., and Chen, Z.
\newblock {Knowledge graph embedding by translating on hyperplanes}.
\newblock In \emph{Twenty-Eighth AAAI conference on artificial intelligence},
  2014.

\bibitem[Yang et~al.(2015)Yang, Yih, He, Gao, and Deng]{yang2014embedding}
Yang, B., Yih, W.-t., He, X., Gao, J., and Deng, L.
\newblock {Embedding Entities and Relations for Learning and Inference in
  Knowledge Bases}.
\newblock In \emph{International Conference on Learning Representations}, 2015.

\bibitem[Yang et~al.(2017)Yang, Yang, and Cohen]{yang2017differentiable}
Yang, F., Yang, Z., and Cohen, W.~W.
\newblock Differentiable learning of logical rules for knowledge base
  reasoning.
\newblock In \emph{Advances in Neural Information Processing Systems}, pp.\
  2319--2328, 2017.

\bibitem[Yu et~al.(2017)Yu, Yu, Fan, and Tao]{yu2017multi}
Yu, Z., Yu, J., Fan, J., and Tao, D.
\newblock {Multi-modal Factorized Bilinear Pooling with Co-attention Learning
  for Visual Question Answering}.
\newblock In \emph{2017 IEEE International Conference on Computer Vision
  (ICCV)}, pp.\  1839--1848. IEEE, 2017.

\end{thebibliography}
\bibliographystyle{icml2020}

\clearpage

\setcounter{table}{0}
\renewcommand{\thetable}{A.\arabic{table}}

\setcounter{figure}{0}
\renewcommand{\thefigure}{A.\arabic{figure}}

\setcounter{equation}{0}
\renewcommand{\theequation}{A.\arabic{equation}}

\setcounter{proposition}{2}
\renewcommand{\theproposition}{\arabic{proposition}}

\icmltitlerunning{LowFER: Low-rank Bilinear Pooling for Link Prediction --- Appendix}

\twocolumn[
\icmltitle{LowFER: Low-rank Bilinear Pooling for Link Prediction --- Appendix}
]

\appendix


\section{Proofs}


\subsection{Proposition 1} \label{appendix:A.1}

\begin{proof}
First, we will prove the case for $k=d_e$, with the proof for the case $k=d_r$ following a similar argument. For both cases, we represent entity embedding vector as $\mathbf{e}_i \in \{0, 1\}^{|\mathcal{E}|}$, such that only $i$-th element is $1$, and similarly, relation embedding vector as $\mathbf{r}_j \in \{0, 1\}^{|\mathcal{R}|}$, such that only $j$-th element is $1$. We represent with $\mathbf{U} \in \mathbb{R}^{d_e \times k d_e}$ and $\mathbf{V} \in \mathbb{R}^{d_r \times k d_e}$ the model parameters, then, given any triple $(e_i, r_j, e_l) \in \mathcal{T}$ with indices $(i, j, l)$, such that $1 \leq i, l \leq |\mathcal{E}|$ and $1 \leq j \leq |\mathcal{R}|$:

For $k=d_e$: We let $\mathbf{U}_{mn} = 1$ for $n = m + (o-1)d_e$, for all $m$ in $\{1, ..., d_e\}$ and for all $o$ in $\{1, ..., k\}$ and $0$ otherwise. Further, let $\mathbf{V}_{pq} = 1$ for $p = j$ and $q = (l-1)d_e + i$ and $0$ otherwise. Applying $\mathbf{g}(e_i, r_j)$ and taking dot product of the resultant vector with $\mathbf{e}_l$ (Eq. \ref{eq:final_scoring_func_lowfer}) perfectly represents the ground truth as $1$. Also, for any triple in $\mathcal{T'}$, a score of $0$ is assigned. 

For $k=d_r$: We let $\mathbf{U}_{mn} = 1$ for $m = i$ and $n = (l-1)d_e + j$ and $0$ otherwise. Further, let $\mathbf{V}_{pq} = 1$ for $q = p + (o-1)d_e$, for all $p$ in $\{1, ..., d_r\}$ and for all $o$ in $\{1, .., k\}$ and $0$ otherwise. Rest of the argument follows the same as for $k=d_e$.
\end{proof}


\subsection{Proposition 2} \label{appendix:A.2}

\begin{proof}
From Eq. \ref{eq:tucker_core_appx} and \ref{eq:tucker_approx}, observe that the $m$-th slice of the core tensor $\mathcal{W}$ on object dimension is approximated by adding $k$ rank-1 matrices, each of which is a cross product between $m$-th column in $\mathbf{W}^{(l)}_{U}$ and $m$-th column in $\mathbf{W}^{(l)}_{V}$, for all $l$ in $\{1, ..., d_e\}$. Each slice of the core tensor $\mathcal{W}$ on object dimension has a maximum rank $\min(d_e, d_r)$ and from Singular Value Decomposition (SVD), there exists $n$ ($\le \min(d_e, d_r)$) scaled left singular and scaled right singular vectors whose sum of the cross products is equal to the slice. By choosing these scaled left singular vectors, scaled right singular vectors and zero vectors (in case the rank of the corresponding slice is less than the maximum rank of any such slice) as columns for matrices $\mathbf{W}^{(l)}_{U}$, $\mathbf{W}^{(l)}_{V}$, for all $l$ in $\{1, ..., d_e\}$,  the core tensor $\mathcal{W}$ is obtained from Eq. \ref{eq:tucker_core_appx} with $k \leq \min(d_e, d_r)$.
\end{proof}

Please note that the bounds presented in Table \ref{table:bounds_summary} are weak and in general, not very useful. They are derived only for checking the full expressibility of a model, which is also referred to as model being \textit{universal} in \citet{wang2018multi}, to handle \textit{all-types} of relations with zero error, i.e., perfect reconstruction of the binary tensor $\mathbf{T}$ for a given $\mathcal{KG}$. Since factorization based methods can be seen as approximating the true binary tensor, more useful bounds can be derived by studying the quality of the approximations for a given accuracy level. The bounds for RESCAL, ComplEx and HolE are reported from \citet{wang2018multi} while for SimplE \cite{kazemi2018simple} and TuckER \cite{balazevic2019tucker}, from their respective papers. 

As discussed in section \ref{sec:fully_expressive}, it was first shown in \citet{wang2018multi} that TransE is not universal, which was later generalized to other translational models by \citet{kazemi2018simple}. RotatE \cite{sun2019rotate}, a state-of-the-art dissimilarity based model, alleviates the issues of TransE by learning counterclockwise rotations in the complex space. For a triple $(h, r, t)$, RotatE models the tail entity as $\mathbf{t} = \mathbf{h} \circ \mathbf{r}$, where $\mathbf{h}, \mathbf{t} \in \mathbb{C}^d$ are head and tail embeddings and $\mathbf{r} \in \mathbb{C}^d$ is the relation embedding with a restriction on the element-wise modulus, $|r_i|=1$. Therefore, it only affects the phases of the entity embeddings in the complex vector space. \citet{sun2019rotate} showed that it can learn \textit{symmetric}, \textit{assymmetric}, \textit{inverse} and \textit{composition} relations (cf. Lemma 1, 2, 3) and degenerates to TransE (cf. Theorem 4). However, we note that RotatE is also not \textit{fully expressive} due to its inability to model the transitive relations in the general case, i.e., irrespective of the size of embedding dimension.

\begin{proposition} \label{proposition3}
RotatE is not fully expressive due to a limitation on the transitive relations.
\end{proposition}

\begin{proof}
Consider $\{e_1, e_2, e_3\} =\Delta \subset \mathcal{E}$ and $r \in \mathcal{R}$ be a transitive relation on $\Delta$ such that $r(e_1, e_2), r(e_2, e_3)$ and $r(e_1, e_3)$ belong to the ground truth. Let $\mathbf{e}_1, \mathbf{e}_2, \mathbf{e}_3, \mathbf{r} \in \mathbb{C}^d$ be the embedding vectors for RotatE. Let us assume that $r(e_1, e_2)$ and $r(e_2, e_3)$ hold with RotatE, then we get $\mathbf{e}_2 = \mathbf{r} \circ \mathbf{e}_1$ and $\mathbf{e}_3 = \mathbf{r} \circ \mathbf{e}_2$. From definition of transitive relation we know that $r(e_1, e_2) \land r(e_2, e_3) \implies r(e_1, e_3)$, here we obtain $\mathbf{e}_3 = \mathbf{r} \circ \mathbf{r} \circ \mathbf{e}_1$. Therefore for $r(e_1, e_3)$ to hold with RotatE, we must have $\mathbf{r} \circ \mathbf{r} = \mathbf{r} \implies \mathbf{r} = \mathbf{1}$, which in turn suggest $\mathbf{e}_1 = \mathbf{e}_2 = \mathbf{e}_3$ but $e_1, e_2, e_3$ are distinct entities. More concretely, this condition requires that for all elements of relation embedding $r_i$, $\text{cos}(\theta_{r,i}) + i\text{sin}(\theta_{r,i})$ should match $\text{cos}(2\theta_{r,i}) + i\text{sin}(2\theta_{r,i})$, which is only possible when $\theta_{r, i} \in \{0, 2\pi\}$, effectively no rotation.
\end{proof}

\begin{figure}[!t]
    \includegraphics[width=0.85\linewidth]{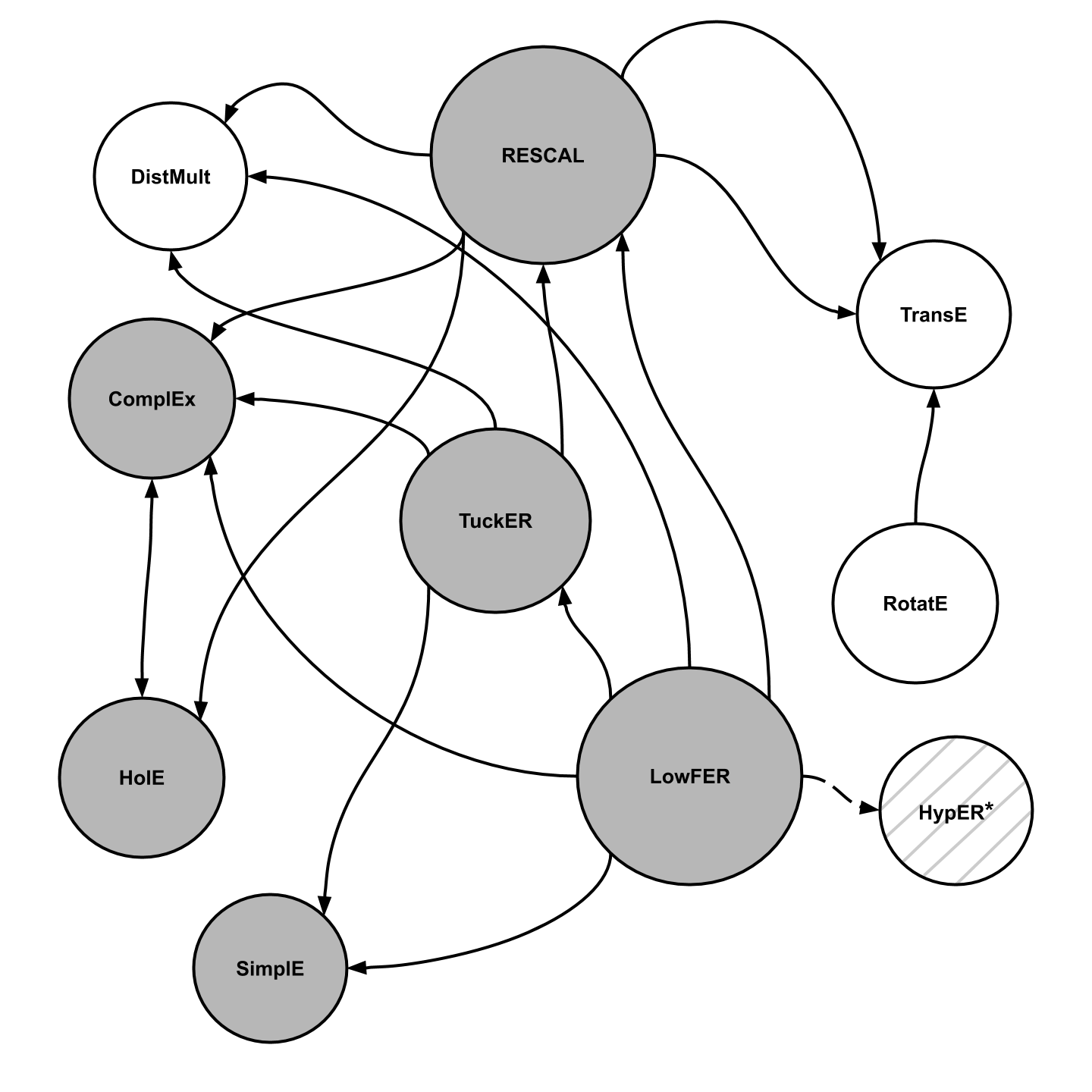}
    \centering
    \caption{\textit{Subsumption map of KGC models for known relationships}: Each node represents a model, where a directed edge shows that the parent node has shown to subsume the child under some conditions. The dotted line shows that the relation is not general enough, where the grey nodes represent \textit{fully expressive} models, the white nodes represent the models that have shown to be not \textit{fully expressive} and dashed ones where this property is not known. The size of a node is relative to the number of outgoing edges. * HypER \cite{balavzevic2019hypernetwork} has shown to be related to factorization based methods up to a non-linearity, but the authors did not specify any explicit modeling subsumption of other models.} \label{fig:scoring_subsumptions}
\end{figure}


\subsection{KGC Scoring Subsumption} \label{sec:subsumptions}

In sections \ref{sec:relation_with_tucker}, \ref{sec:relations_with_bilinear_models} and \ref{sec:relation_with_hyper} we presented LowFER's relations to other models. In this section, we breifly summarize the subsumption findings of related works. Please note that we only discuss the published findings and refrain from any implied results.

First, \citet{hayashi2017equivalence} showed the equivalence between ComplEx and HolE up to a constant factor using Parseval's theorem\footnote{For $\mathbf{x}, \mathbf{y} \in \mathbb{R}^d$, it states that $\mathbf{x}^T\mathbf{y}=\frac{1}{d} \mathcal{F}(\mathbf{x})^T \overline{\mathcal{F}(\mathbf{y})}$, where $\mathcal{F}: \mathbb{R}^d \rightarrow \mathbb{C}^d$ is the discrete Fourier transform (DFT).}, which was also discussed in \citet{trouillon2017complex}. Then, the key contributions came from the work of \citet{wang2018multi}, who showed that RESCAL subsumes TransE, ComplEx, HolE and DistMult by the arguments of ranking tensor. \citet{kazemi2018simple} presented a unified understanding of RESCAL, DistMult, ComplEx and SimplE as \textit{family of bilinear models} under different constraints on the bilinear map. In contrast to the black box 2D-convolution based ConvE model, HypER \cite{balavzevic2019hypernetwork} showed that 1D-convolution with \textit{hypernetworks} \cite{ha2016hypernetworks} come close to well established factorization based methods up to a non-linearity. \citet{balazevic2019tucker} showed that with certain constraints on the core tensor of the Tucker decomposition \cite{tucker1966some}, it can subsume the \textit{family of bilinear models}. In this work, we showed that LowFER subsumes TuckER and can be seen as providing low-rank approximation of the core tensor\footnote{The rank of a tensor is the minimal number of rank-$1$ tensors that yield it in a linear combination. It is known that the tensor rank is NP-hard to compute, and for a 3rd-order tensor $n \times m \times k$, it can be more than $\min(n, m, k)$ but no more than $\min(nm, nk, mk)$ \cite{miettinen2011boolean}. Whereas, the $n$-rank of a tensor $\mathcal{W}$ is the dimension of the vector space spanned by the $n$-mode vectors, which are the columns of the matrix unfolding $\mathbf{W}_{(n)}$ \cite{de2000multilinear}.} with accurate representation under certain conditions (Proposition \ref{proposition2}). We also showed that LowFER can subsume the \textit{family of bilinear models} and HypER up to a non-linearity. Figure \ref{fig:scoring_subsumptions}\footnote{\url{https://bit.ly/3k641Ba}} provides a network style map for the models discussed here.


\section{Experiments} \label{appendix:B}

In this section, we will present the details of the datasets, evaluation metrics, model implementation, the choice of hyperparameters and report additional experiments.


\subsection{Data}

We conducted the experiments on four benchmark datasets: WN18 \cite{bordes2013translating} - a subset of Wordnet, WN18RR \cite{dettmers2018convolutional} - a subset of WN18 created through the removal of inverse relations from validation and test sets, FB15k \cite{bordes2013translating} - a subset of Freebase, and FB15k-237 \cite{toutanova2015representing} - a subset of FB15k created through the removal of inverse relations from validation and test sets. Table \ref{tab:datasets} shows the statistics of all the datasets.


\begin{table}[!t]
    \centering
    \caption{Datasets used for link prediction experiments, where $n_e$=number of entities, $n_r$=number of relations and the entities-to-relations ratio $n_e/n_r$ is approximated to the nearest integer.}\label{tab:datasets}
    \vspace{0.3cm}
    \resizebox{8cm}{!}{
      \begin{tabular}{lrrrrrr}
        \toprule 
        Dataset & $n_e$ & $n_r$ & $n_e/n_r$ & Training & Validation & Testing \\
        \midrule 
        WN18 & $40,943$ & $18$ & $2275$ & $141,442$ & $5,000$ & $5,000$ \\
        WN18RR & $40,943$ & $11$ & $3722$ & $86,835$ & $3,034$ & $3,134$ \\
        FB15k & $14,951$ & $1,345$ & $11$ & $483,142$ & $50,000$ & $59,071$ \\
        FB15k-237 & $14,541$ & $237$ & $61$ & $272,115$ & $17,535$ & $20,466$ \\
        \bottomrule
      \end{tabular}
    }
\end{table}


\subsection{Evaluation Metrics}

We report the standard metrics of Mean Reciprocal Rank (MRR) and Hits@$k$ for $k \in \{1, 3, 10\}$. For each test triple $(e_s, r, e_o)$, we score all the triples $(e_s, r, e)$ for all $e \in \mathcal{E}$. We then compute the inverse rank of true triple and average them over all examples. However, \citet{bordes2013translating} identified an issue with this evaluation and introduced \textit{filtered} MRR, where we only consider triples of the form $\{(e_s, r, e) \; | \; \forall e \in \mathcal{E} \ \text{s.t.} \; (e_s, r, e) \not\in \text{train} \cup \text{valid} \cup \text{test}\}$ during evaluation. We therefore reported \textit{filtered} MRR for all the experiments. The Hits@$k$ metric computes the percentage of test triples whose ranking is less than or equal to $k$.


\begin{table}[!t]
    \centering
    \caption{Best performing hyper-parameter values for LowFER, where lr=learning rate, dr=decay rate, $d_e$=entity embedding dimension, $d_r$=relation embedding dimension, $k$=LowFER factorization rank, dE=entity embedding dropout, dMFB=MFB dropout, dOut=output dropout and ls=label smoothing. Please note that dE, dMFB and dOut are the same as d\#$1$, d\#$2$ and d\#$3$ as in TuckER (see Appendix A in \citet{balazevic2019tucker}) respectively.}\label{table:hyperparams}
    \vspace{0.3cm}
    \resizebox{8cm}{!}{
    \begin{tabular}{lccccccccc}
      \toprule
      Dataset & lr & dr & $d_e$ & $d_r$ & $k$ & dE  & dMFB & dOut & ls\\
      \midrule
      WN18 & $0.005$ & $0.995$ & $200$ & $30$ & $10$ & $0.2$ & $0.1$ & $0.2$ & $0.1$ \\
      WN18RR & $0.01$ & $1.0$ & $200$ & $30$ & $30$ & $0.2$ & $0.2$ & $0.3$ & $0.1$ \\
      FB15k & $0.003$ & $0.99$ & $300$ & $30$ & $50$ & $0.2$ & $0.2$ & $0.3$ & $0.0$ \\
      FB15k-237 & $0.0005$ & $1.0$ & $200$ & $200$ & $100$ & $0.3$ & $0.4$ & $0.5$ & $0.1$ \\
      \bottomrule
    \end{tabular}
    }
\end{table}


\subsection{Implementation and Hyperparameters}

We implemented LowFER\footnote{\url{https://github.com/suamin/LowFER}} using the open-source code released by TuckER \cite{balazevic2019tucker}\footnote{\url{https://github.com/ibalazevic/TuckER}}. We did random search over the embedding dimensions in $\{30, 50, 100, 200, 300\}$ for $d_e$ and $d_r$. Further, we varied the factorization rank $k$ in $\{1, 5, 10, 30, 50, 100, 150, 200\}$, with $k=1$ (LowFER-$1$) and $k=10$ (LowFER-$10$) as baselines. For WN18RR and WN18, we found best $d_e=200$ and $d_r=30$ with $k$ value of $30$ and $10$ respectively. For FB15k-237, we found best $d_e=d_r=200$ at $k=100$. All of these embedding dimensions match the best reported in TuckER \cite{balazevic2019tucker}. However, for FB15k, we found using the configuration of $d_e=300$ and $d_r=30$ to be consistently better than $d_e=d_r=200$. For fair comparison, we also reported the results for $d_e=d_r=200$ and the best configuration when $d_e=200$ and $(d_r,k)\leq200$ (Table \ref{tab:fb15k_200_50_150}).

Similar to \citet{balazevic2019tucker}, we used Batch Normalization \cite{ioffe2015batch} but additionally power normalization and $l_2$-normalization to stabilize training from large outputs following the Hadamard product in main scoring function \cite{yu2017multi}\footnote{We observed no performance degradation by removing these additional normalization techniques but we used it in all the experiments to be consistent with prior work of \citet{yu2017multi}.}. We tested the best reported hyperparmeters of \citet{balazevic2019tucker} with random search and observed good performance in initial testing. With $d_e$, $d_r$ and $k$ selected, we used fixed set of values for rest of the hyperparameters reported in \citet{balazevic2019tucker}, including learning rate, decay rate, entity embedding dropout, MFB dropout, output dropout and label smoothing \cite{szegedy2016rethinking,pereyra2017regularizing} (see Table \ref{table:hyperparams} for the best hyperparameters). We used Adam \cite{kingma2014adam} for optimization. In all the experiments, we trained the models for $500$ epochs with batch size $128$ and reported the final results on test set.


\subsection{Results on YAGO3-10}

We report additional results on YAGO3-10, which is a subset of YAGO3 \cite{mahdisoltani2013yago3}, consisting of $123,182$ entities and $37$ relations such that have each entity has at least $10$ relations. We used the same best hyperparameters as for WN18RR. Table \ref{tab:yago10_results} shows that our model outperforms state-of-the-art models including RotatE and HypER. It is worth noting that LowFER-$k$* on YAGO3-10 has only $\sim$26M parameters compared to $\sim$61M parameters of RotatE \cite{sun2019rotate}, which also includes their self-adversarial negative sampling.
 

\begin{table}[!t]
    \centering
    \caption{Link prediction results on YAGO3-10. Results for DistMult, ComplEx and ConvE are taken from \citet{dettmers2018convolutional} and for RotatE \cite{sun2019rotate} (with self-adversarial negative sampling) and HypER \cite{balavzevic2019hypernetwork} are taken from respective papers.}\label{tab:yago10_results}
    \vspace{0.3cm}
    \resizebox{7.5cm}{!}{
    \begin{tabular}{lcccc}
    \toprule 
    Model & MRR & Hits@1 & Hits@3 & Hits@10\\
    \midrule
    DistMult & $0.340$ & $0.240$ & $0.380$ & $0.540$ \\
    ComplEx & $0.360$ & $0.260$ & $0.400$ & $0.550$ \\
    ConvE & $0.440$ & $0.350$ & $0.490$ & $0.620$ \\
    RotatE & $0.495$ & $0.402$ & $0.550$ & $0.670$ \\
    HypER & $\underline{0.533}$ & $\underline{0.455}$ & $\underline{0.580}$ & $\underline{0.678}$ \\
    LowFER-$k$* & $\mathbf{0.537}$ & $\mathbf{0.457}$ & $\mathbf{0.583}$ & $\mathbf{0.688}$ \\
    \bottomrule
    \end{tabular}
    }
 \end{table}


\begin{table}[!t]
    \centering
    \caption{Link prediction results with LowFER-$k$* and additional \texttt{tanh} non-linearity. The $\downarrow$ shows that the performance went down compared to the linear counterparts reported in Table \ref{table:main_results}.}\label{tab:non_linear_lowfer_results}
    \vspace{0.3cm}
    \resizebox{7.5cm}{!}{
    \begin{tabular}{lcccc}
    \toprule 
    Dataset & MRR & Hits@1 & Hits@3 & Hits@10\\
    \midrule
    FB15k-237 $\downarrow$ & $0.345$ & $0.256$ & $0.378$ & $0.526$ \\
    FB15k $\downarrow$ & $0.818$ & $0.771$ & $0.850$ & $0.898$ \\
    WN18RR $\downarrow$ & $0.457$ & $0.429$ & $0.469$ & $0.511$ \\
    WN18 & $0.950$ & $0.946$ & $0.952$ & $0.957$ \\
    \bottomrule
    \end{tabular}
    }
 \end{table}
 

\subsection{LowFER with Non-linearity}

Similar to \citet{kim2016hadamard}, we perform a simple ablation study by adding non-linearity to the LowFER scoring function as follows:
\begin{equation*}
    \bar{f}(e_s, r, e_o) = (\sigma(\mathbf{S}^k\text{diag}(\mathbf{U}^T \mathbf{e}_s)\mathbf{V}^T \mathbf{r}))^{T}\mathbf{e}_{o}
\end{equation*}
\noindent where we use hyperbolic tangent $\sigma=$ \texttt{tanh} non-linearity. Applying non-linear activation function can be seen as increasing the representation capacity of the model but Table \ref{tab:non_linear_lowfer_results} shows that the general performance of LowFER goes down.


\subsection{Models Comparison}

We compared LowFER with non-linear models including ConvE \cite{dettmers2018convolutional}, R-GCN \cite{schlichtkrull2018modeling}, Neural LP \cite{yang2017differentiable}, RotatE \cite{sun2019rotate}\footnote{Where we reported their results in Table \ref{table:main_results} without the self-adversarial negative sampling. For fair comparison, see Appendix H in their paper.}, TransE \cite{bordes2013translating}, TorusE \cite{ebisu2018toruse} and HypER \cite{balavzevic2019hypernetwork}. In linear models, we compared against DistMult \cite{yang2014embedding}, HolE \cite{nickel2016holographic}, ComplEx \cite{trouillon2016complex}, ANALOGY \cite{liu2017analogical}, SimplE \cite{kazemi2018simple} and state-of-the-art TuckER \cite{balazevic2019tucker} model. Results for the Canonical Tensor Decomposition \cite{lacroix2018canonical} were not included due to the uncommon choice of extremely large embedding dimensions of $d_e=d_r=2000$.

Additional models that were not reported in the main results (Table \ref{table:main_results}) due to partial results but were still outperformed by LowFER include M-Walk \cite{shen2018m} with their reported metrics of MRR=$0.437$, Hits@1=$0.414$ and Hits@3=$0.445$ on WN18RR and MINERVA \cite{das2017go} with Hits@10=$0.456$ on FB15k-237. The results in Table \ref{table:main_results} for all the models were taken from \citet{balavzevic2019hypernetwork} and \citet{balazevic2019tucker}. Lastly, in the section \ref{sec:per_rel_analysis}, to perform per relations comparisons, we trained the TuckER models with the best reported configurations in \citet{balazevic2019tucker} for WN18 and WN18RR.

\end{document}